\documentclass[journal]{IEEEtran}
\usepackage{amsmath,amsfonts}
\usepackage{amsthm}
\usepackage{array}
\usepackage[caption=false,font=normalsize,labelfont=sf,textfont=sf]{subfig}
\usepackage{textcomp}
\usepackage{stfloats}
\usepackage{url}
\usepackage{verbatim}
\usepackage{graphicx}
\usepackage{cite}
\usepackage{booktabs}
\usepackage{multirow}
\usepackage{graphicx}

\usepackage{makecell}

\usepackage[ruled]{algorithm2e}

\SetAlgoNoLine        
\SetAlgoNoEnd         
\DontPrintSemicolon   

\SetNlSty{}{}{:\ }

\usepackage{bm}
\usepackage{hyperref}
\usepackage[table]{xcolor}
\usepackage{tabularx}

\theoremstyle{remark}

\makeatletter
\def\th@remark{%
  \normalfont 
  \thm@headfont{\bfseries\itshape}
}
\makeatother

\theoremstyle{definition}
\newtheorem{assumption}{Assumption}

\theoremstyle{definition}

\theoremstyle{plain} 
\newtheorem{proposition}{Proposition}

\begin{document}

\title{EquiForm: Noise-Robust $\mathrm{SE}(3)$-Equivariant Policy Learning from 3D Point Clouds}

\author{Zhiyuan Zhang, Yu She$^{\dagger}$
\thanks{$^{\dagger}$ Corresponding Author.}
\thanks{Zhiyuan Zhang and Yu She are with the School of Industrial Engineering,
Purdue University, West Lafayette, USA  
{\tt\footnotesize zhan5570@purdue.edu; shey@purdue.edu}}%
}

\maketitle

\begin{abstract}
Visual imitation learning with 3D point clouds has advanced robotic manipulation by providing geometry-aware, appearance-invariant observations. 
However, point cloud–based policies remain highly sensitive to sensor noise, pose perturbations, and occlusion-induced artifacts, which distort geometric structure and break the equivariance assumptions required for robust generalization.
Existing equivariant approaches primarily encode symmetry constraints into neural architectures, but do not explicitly correct noise-induced geometric deviations or enforce equivariant consistency in learned representations.
We introduce EquiForm, a noise-robust $\mathrm{SE}(3)$-equivariant policy learning framework for point cloud–based manipulation.
EquiForm formalizes how noise-induced geometric distortions lead to equivariance deviations in observation-to-action mappings, and introduces a geometric denoising module to restore consistent 3D structure under noisy or incomplete observations.
In addition, we propose a contrastive equivariant alignment objective that enforces representation consistency under both rigid transformations and noise perturbations.
Built upon these components, EquiForm forms a flexible policy learning pipeline that integrates noise-robust geometric reasoning with modern generative models. We evaluate EquiForm on 16 simulated tasks and 4 real-world manipulation tasks across diverse objects and scene layouts. Compared to state-of-the-art point cloud imitation learning methods, EquiForm achieves an average improvement of 17.2\% in simulation and 28.1\% in real-world experiments, demonstrating strong noise robustness and spatial generalization.
For more details, please refer to the project website: \url{https://zhangzhiyuanzhang.github.io/equiform-website/}.
\end{abstract}


\section{Introduction}
Imitation learning has made remarkable progress in robotic manipulation, enabling robots to acquire complex behaviors directly from human demonstrations \cite{diffusion_policy,zhao2023learning,DP3,o2024open}. Despite this progress, generalization and sample efficiency remain key challenges. Visual imitation learning policies \cite{diffusion_policy,zhao2023learning} often overfit to the limited distribution of training demonstrations and struggle to generalize to unseen variations in objects, scene layouts, or camera viewpoints.
Small distribution shifts such as changes in object poses, lighting, or backgrounds can lead to significant performance degradation.

To address the limitations of image-based policies, recent works have explored the use of 3D point clouds as policy inputs \cite{DP3,wang2024gendp,huang20243d,zhang2025canonical}.
Point clouds provide a geometry-aware, appearance-invariant representation that is more robust to visual distractors such as color, texture, and background clutter.
This geometric abstraction enables policies to better capture the relationship between environment structure and robot actions, improving generalization to diverse object shapes and spatial configurations\cite{DP3}. In real-world settings, such point cloud representations arise naturally in robotic manipulation tasks, as illustrated in Fig.~\ref{fig:system_overview}.

\begin{figure}[t]
    \centering
    \includegraphics[width=1.0\linewidth]{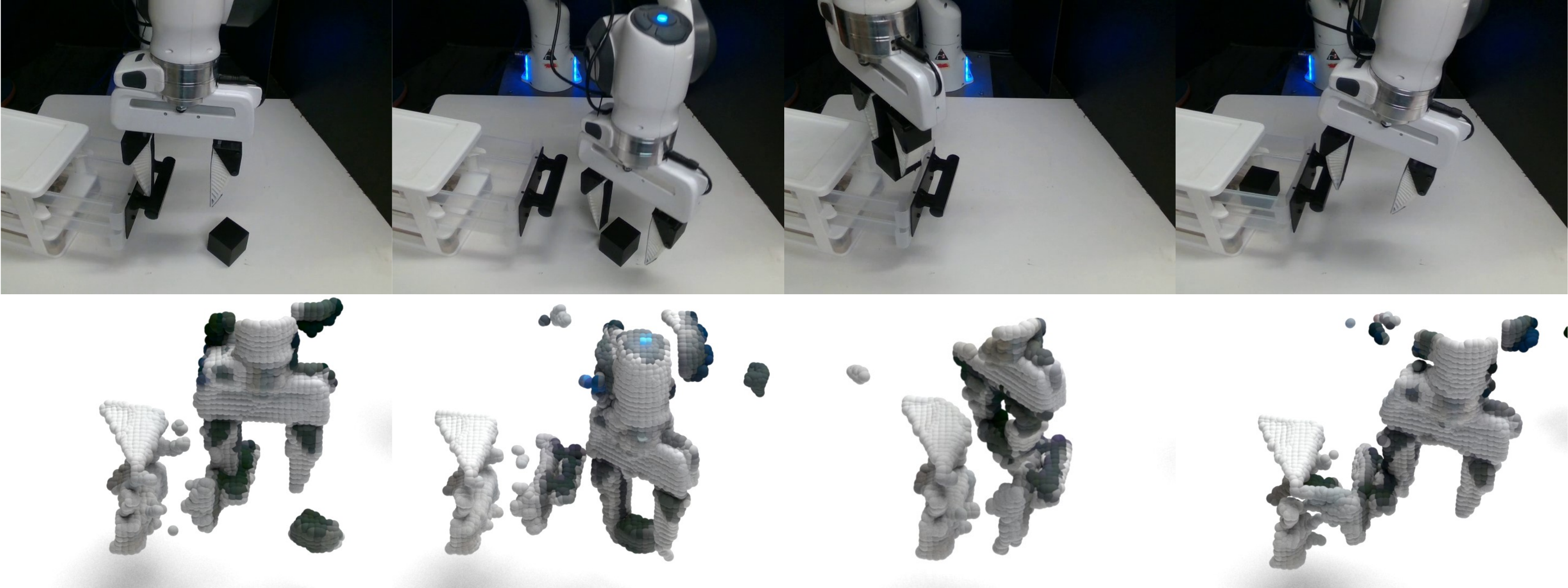}
    \caption{Real-world robotic manipulation with point cloud observations. Image observations are shown in the top row, and the corresponding point cloud observations are shown in the bottom row.}
    \label{fig:system_overview}
\end{figure}
While point clouds improve robustness to visual appearance, they remain
challenging to generalize across diverse scene layouts~\cite{zhang2025canonical}.
Such variations in object orientations and relative poses require the
policy to act consistently under different geometric transformations.
This motivates the use of $\mathrm{SE}(3)$ equivariance, which captures the rotational and translational symmetries inherent in many manipulation tasks. By ensuring that the policy responds predictably to rigid transformations of the input, $\mathrm{SE}(3)$-equivariant architectures can significantly enhance sample efficiency and spatial generalization in 3D manipulation.\cite{equidiff,actionflow,EquiAct,Equibot,zhang2025canonical}.

Current equivariant policy learning approaches typically assume clean and stable geometric observations.
In practice, point clouds often contain depth noise, occlusions, missing surface regions, and inconsistencies introduced by sparse, irregular, or discretized point sampling~\cite{qi2017pointnet, pointnet++, xie2020pointcontrast}.
Such noise introduces geometric distortions that disrupt the correspondence between rigid transformations in input space and the induced transformations in feature space, leading to equivariance deviations that degrade policy performance.
These failure modes are illustrated in Fig.~\ref{fig:abstract}, where non-equivariant policies suffer from misalignment, and noise-sensitive equivariant policies exhibit unstable behaviors under noisy point cloud observations, despite performing well under idealized conditions.
However, existing methods do not model how noise affects $\mathrm{SE}(3)$ symmetry or provide mechanisms to preserve equivariance under unreliable geometry.

To address these challenges, we propose EquiForm, a noise-robust $\mathrm{SE}(3)$-equivariant policy learning framework designed to operate reliably on noisy point cloud observations.
Our contributions are summarized as follows:

1. We provide a formal analysis of how noise-induced geometric distortions affect point cloud representations and break $\mathrm{SE}(3)$ equivariance in observation-to-action mappings. Building on this understanding, we introduce a geometric denoising module that restores stable 3D structure under noisy observations, providing a reliable foundation for downstream equivariant reasoning.

\begin{figure}[t]
    \centering
    \includegraphics[width=1.0\linewidth]{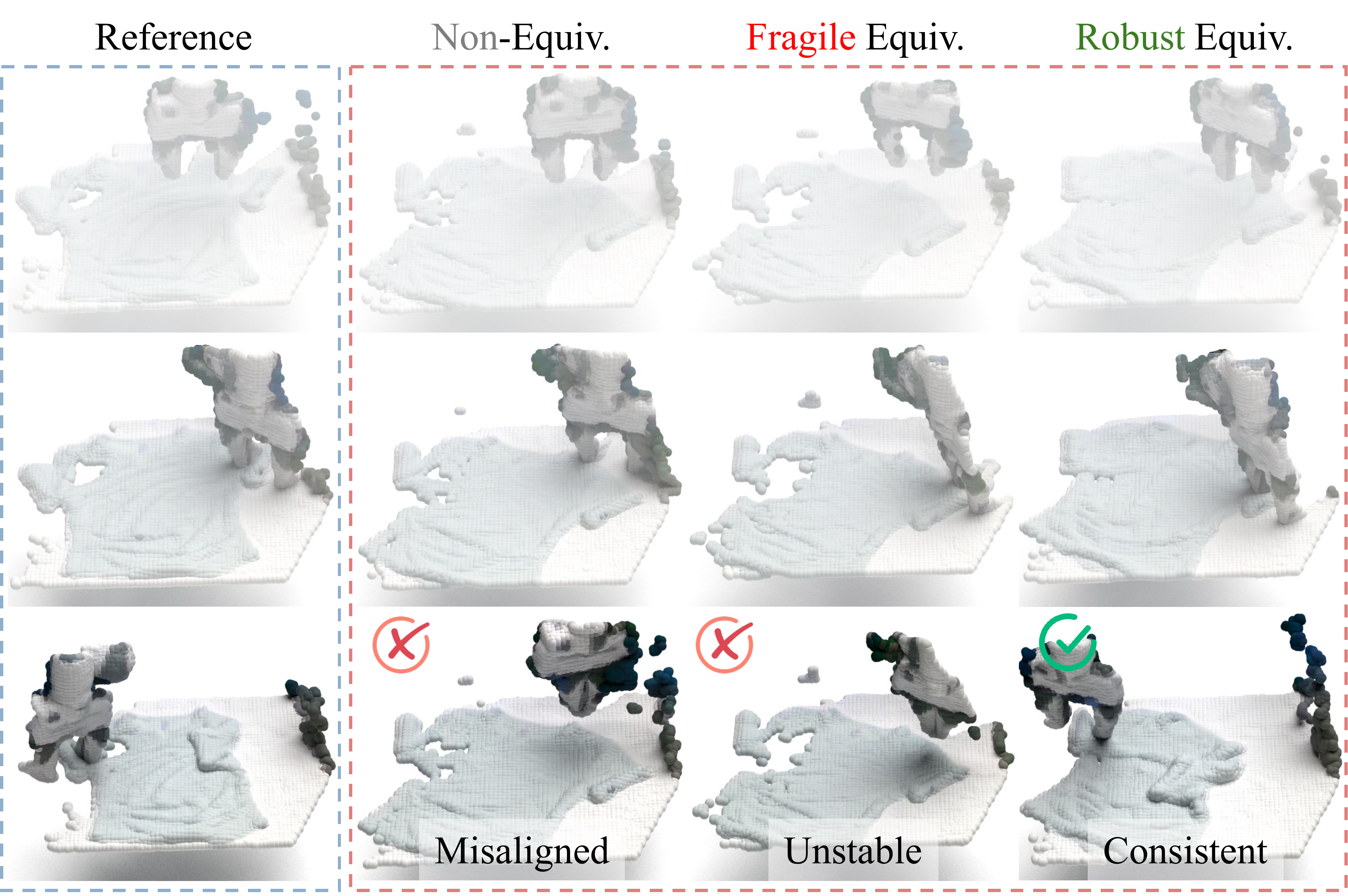}
    \caption{Comparison of non-equivariant, fragile equivariant, and robust equivariant policies under noisy point cloud observations. Robust equivariance produces consistent behaviors across task progressions, while non-equivariant or fragile equivariant policies suffer from misalignment or instability.}
    \label{fig:abstract}
\end{figure}
2. We propose a contrastive equivariant alignment loss that encourages consistent 3D representations under both rigid transformations and noise perturbations. This objective explicitly promotes noise-robust equivariance, enabling the learned representation to maintain symmetry.

3. We evaluate EquiForm in both simulation and real-world settings.
Across 16 simulated tasks and 4 real-world manipulation tasks, spanning varying noise levels and diverse object instances and scene layouts, EquiForm achieves strong performance relative to existing point cloud-based imitation learning methods, demonstrating clear improvements in noise robustness and spatial generalization.

\section{Related Works}

\subsection{Point Cloud Policy Learning}
Imitation learning for robotic manipulation has evolved from low-dimensional state representations to high-dimensional visual observations, enabling the execution of complex behaviors in unstructured environments \cite{diffusion_policy,zhao2023learning,chi2024universal}. While 2D image-based policies have demonstrated strong performance, their generalization is often constrained by sensitivity to illumination changes, texture variations, and camera viewpoint shifts. 

To mitigate these limitations, recent work has increasingly turned to 3D point cloud representations \cite{DP3,iDP3,wang2024gendp,huang20243d,xue2025demogen,zhang2025canonical}. By explicitly encoding spatial geometry and abstracting away appearance-related distractors, point clouds allow policies to ground decision-making in the structural relationship between the robot and its environment. Leading approaches in this domain, such as DP3 \cite{DP3} and its extensions \cite{iDP3,wang2024gendp,xue2025demogen,chisari2024learning}, integrate point cloud encoders with expressive generative backbones, including diffusion models \cite{diffusion_policy} and flow matching \cite{flow_matching}, to capture multi-modal action distributions.

However, these approaches typically assume high-fidelity geometric observations without considering potential distortions.
In practice, point clouds are frequently sparse, non-uniform, and corrupted by depth noise, occlusions, and sampling artifacts \cite{huang20243d}, causing test-time geometry to deviate from the idealized training distribution.
Policies trained on near-pristine data often suffer marked performance degradation under such conditions.
Our work builds upon these 3D policy foundations but explicitly incorporates a geometric denoising mechanism to improve robustness under sensing imperfections.

\subsection{Equivariance in Robot Learning}
\label{sec:related_equiv}
Robotic manipulation is intrinsically grounded in 3D Euclidean space, where many tasks satisfy $\mathrm{SE}(3)$ symmetries: rigid transformations of the scene should induce corresponding transformations of the robot's action.
Exploiting this structure, $\mathrm{SE}(3)$-equivariant policy learning has emerged as a powerful paradigm for enhancing sample efficiency and spatial generalization \cite{simeonov2023se,wang22onrobot,huang2023leveraging,equidiff,actionflow,EquiAct,Equibot,gao2024riemann,zhang2025canonical}.

Existing methods fall broadly into three categories. The first class enforces equivariance through the network architecture itself\cite{Equibot,EquiAct,simeonov2023se,gao2024riemann,equidiff}, using modules such as Vector Neurons\cite{VN}, Tensor Field Networks\cite{thomas2018tensor}, or Steerable Equivariant CNNs\cite{cesa2022program}. These approaches guarantee symmetry by construction but often require specialized layers and are tightly coupled to specific model designs.
The second class achieves equivariance implicitly by operating in a relative or invariant space, as exemplified by ActionFlow\cite{actionflow} and Invariant Point Attention\cite{ipa}. These methods encode features using invariant interactions in local reference frames and then reconstruct global actions through local-to-global update rules, yielding equivariant behavior without explicit frame estimation.
The third class derives equivariance through canonicalization \cite{zhang2025canonical,ma2024canonicalization,kaba2023equivariance}, which maps each observation to a standardized pose before policy inference. By explicitly estimating a canonical $\mathrm{SE}(3)$ frame and transforming both observations and actions into this shared space, canonicalization offers strong interpretability, ensures coherent multi-modal alignment, and integrates seamlessly with model-agnostic generative policy heads.

While each paradigm offers distinct advantages, most equivariant formulations implicitly assume stable and accurate geometry. In practice, depth noise, partial observations, and irregular sampling introduce geometric distortions that prevent observations from strictly following ideal rigid-body transformations. This leads to equivariance deviation, where input perturbations break the theoretical transformation rules of the equivariant architecture and result in inconsistent policy behavior. Unlike prior works that enforce equivariance under idealized conditions, EquiForm combines explicit canonicalization with noise-aware modeling, aiming to preserve $\mathrm{SE}(3)$-consistent control even when the underlying geometric structure is perturbed.

\subsection{Noise-Robust Learning for 3D Perception}
Robustness to sensor noise and sampling irregularities is a key requirement for deploying 3D perception systems in realistic environments. In the broader 3D vision community, point cloud denoising has been approached through geometric filtering, statistical outlier removal, and learning-based reconstructions that recover clean surfaces or robust features \cite{ren2021overall,gao2022reflective,wang2023fcnet,zhang2023particle,pointnet++,qi2017pointnet}. 

In robotic manipulation, however, noise handling is often treated as a decoupled pre-processing stage or is implicitly left to the robustness of generic encoders \cite{DP3,iDP3,Equibot,zhang2025canonical}.
Such strategies are limited in practice: filtering can discard subtle geometric cues that are crucial for manipulation, and unconstrained encoders lack mechanisms to preserve the $\mathrm{SE}(3)$ structure necessary for reliable generalization under noise.

In contrast, EquiForm introduces a unified framework that jointly addresses noise robustness and equivariant representation learning. Rather than treating denoising as a separate preprocessing step or auxiliary objective, our approach integrates a geometric denoising module with a contrastive equivariant alignment loss. This combination shapes a latent space that is resistant to local sensor artifacts while remaining equivariant to global rigid motions. By simultaneously accounting for noise and symmetry, EquiForm encourages the policy to rely on stable and task-relevant geometric structure, improving reliability under realistic, noisy operating conditions.

\section{Preliminaries}
This section establishes the theoretical foundations of EquiForm.
We first formalize equivariance in the context of robotic policy learning.
We then review $\mathrm{SE}(3)$ canonicalization, which transforms observations into a standardized reference frame to enable spatial generalization.
Finally, we analyze how sensor noise disrupts geometric symmetries and leads to equivariance deviation, motivating the noise-robust mechanisms proposed in our method.

\subsection{Equivariance in Policy Learning}
\begin{figure}[t]
    \centering
    \includegraphics[width=0.9\linewidth]{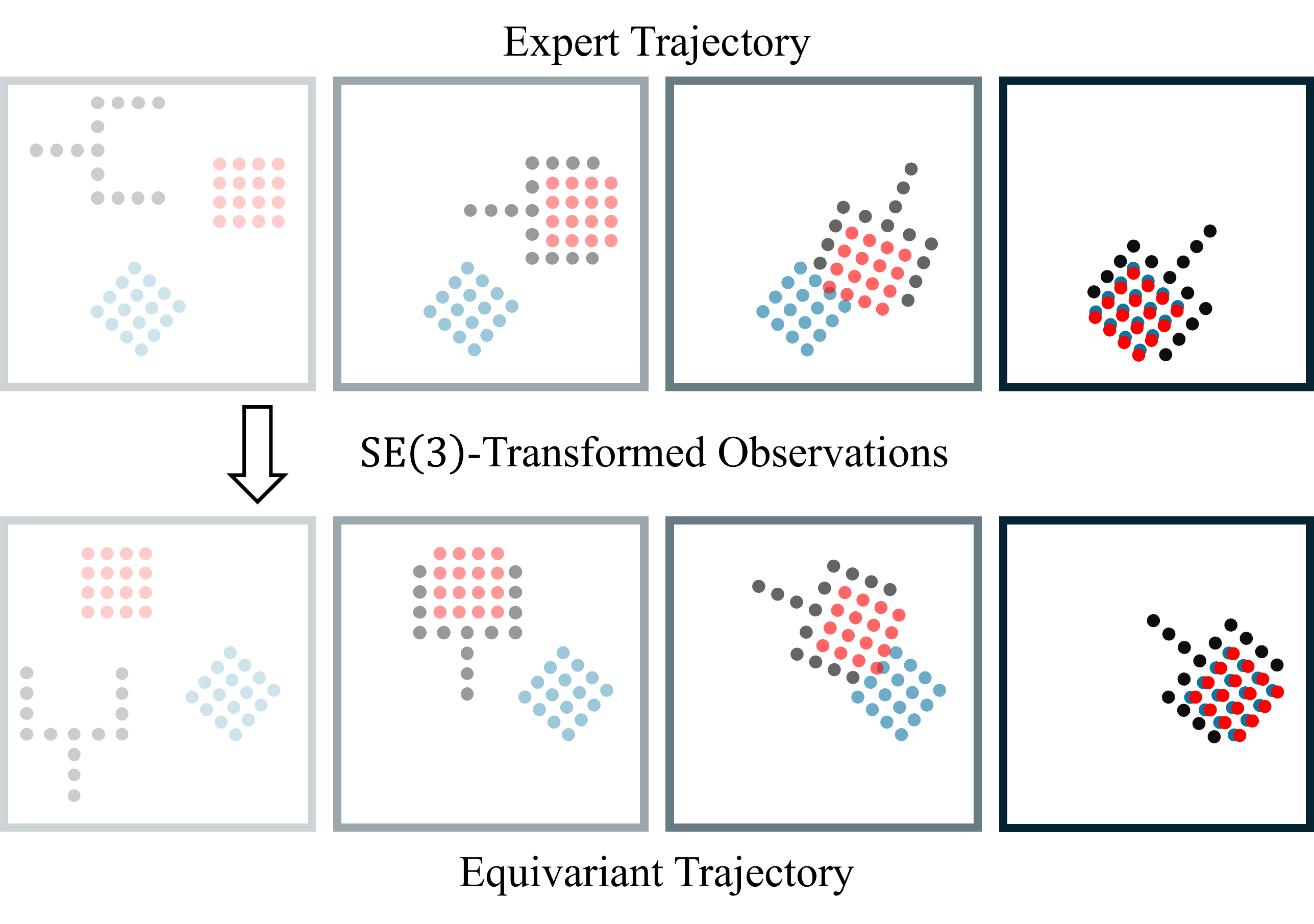}
    \caption{A rigid transformation of the scene observation induces a corresponding transformation of the expert action.}
    \label{fig:equi_traj}
\end{figure}
We consider the problem of learning a robotic control policy $\pi$ via behavior cloning. The objective is to learn a function that maps a history of observations $\mathcal{O} = \{O_{t-m+1}, \ldots, O_t\}$ to a sequence of future actions $\mathcal{A} = \{\mathbf{a}_t, \ldots, \mathbf{a}_{t+n-1}\}$:
\begin{equation*}
\pi: \mathcal{O} \rightarrow \mathcal{A},
\end{equation*}
such that the robot imitates the behavior demonstrated by an expert.
Each observation may include both visual inputs (e.g. images, voxels, or point clouds) and proprioceptive information (e.g. end-effector pose or joint angles).
Since robotic manipulation takes place in 3D space and the resulting scene interactions are governed by rigid-body transformations, it is natural to represent visual input using point clouds and to express proprioception and actions in end-effector pose space.
In the following, we denote each observation as $O=\{X, \mathbf{s}\}$, 
where $X \in \mathbb{R}^{3 \times N}$ is a point cloud composed of $N$ points, and $\mathbf{s}$ represents the robot's end-effector pose (e.g., position, orientation, and gripper width). The corresponding action $\mathbf{a}$ is expressed in the same pose space as an end-effector command.

In 3D manipulation tasks, interactions are typically invariant to the choice of reference frame. Consequently, the optimal policy is expected to satisfy an equivariance property: if the scene undergoes a transformation, the corresponding action should transform in a consistent manner.
Formally, let $g$ denote a general transformation, and let $g\triangleright_\mathcal{O} \mathcal{O}$ and $g\triangleright_\mathcal{A} \mathcal{A}$ denote its action on the observation and action spaces, respectively.
An equivariant policy satisfies:
\begin{equation*}
    \pi(g\triangleright_\mathcal{O} \mathcal{O})=g\triangleright_\mathcal{A} \pi(\mathcal{O}).
\end{equation*}
While the policy is conditioned on a temporal sequence of observations, the equivariance constraint is intrinsic to the spatial structure of individual observations. Therefore, for notational clarity, we formulate this property on a single time step.
In robotic manipulation, a rigid transformation applied to the scene observation necessitates a corresponding transformation of the action to maintain physical consistency, as illustrated in Fig.~\ref{fig:equi_traj}.
This motivates the following assumption:
\begin{assumption}
\label{asp:T}
Expert demonstrations are high-quality and consistent with the underlying task symmetry.
Specifically, if an observation undergoes a rigid transformation $T\in\mathrm{SE}(3)$, the corresponding expert action transforms accordingly:
\begin{equation*}
    \pi^{\star}(TO)=T \pi^{\star}(O),
\end{equation*}
where $\pi^\star$ denotes the idealized expert policy.
\end{assumption}

This assumption reflects the fact that many manipulation tasks are defined up to rigid transformations, such that the desired action should transform equivariantly with the observation\cite{EquiAct, Equibot, zhang2025canonical,equidiff}. Under Assumption~\ref{asp:T}, we conceptually do not distinguish between the transformation acting on the observation and the one acting on the action space.
Furthermore, since future action decisions are conditioned on the current point cloud observation, it is natural to extract transformation-related information directly from the observation itself, as described in the following subsection.

\subsection{SE(3) Canonicalization Policy Learning}
\label{sec:can_policy_learning}
As introduced in Section~\ref{sec:related_equiv}, $\mathrm{SE}(3)$ canonicalization aims to explicitly estimate the rigid transformation that maps a given observation to a canonical frame. Point clouds that differ only by a rigid motion yield identical representations, enabling the policy to learn consistently without being affected by pose variation~\cite{zhang2025canonical}. We now provide the mathematical formulation of this process.

\begin{figure}[t]
    \centering
    \includegraphics[width=0.9\linewidth]{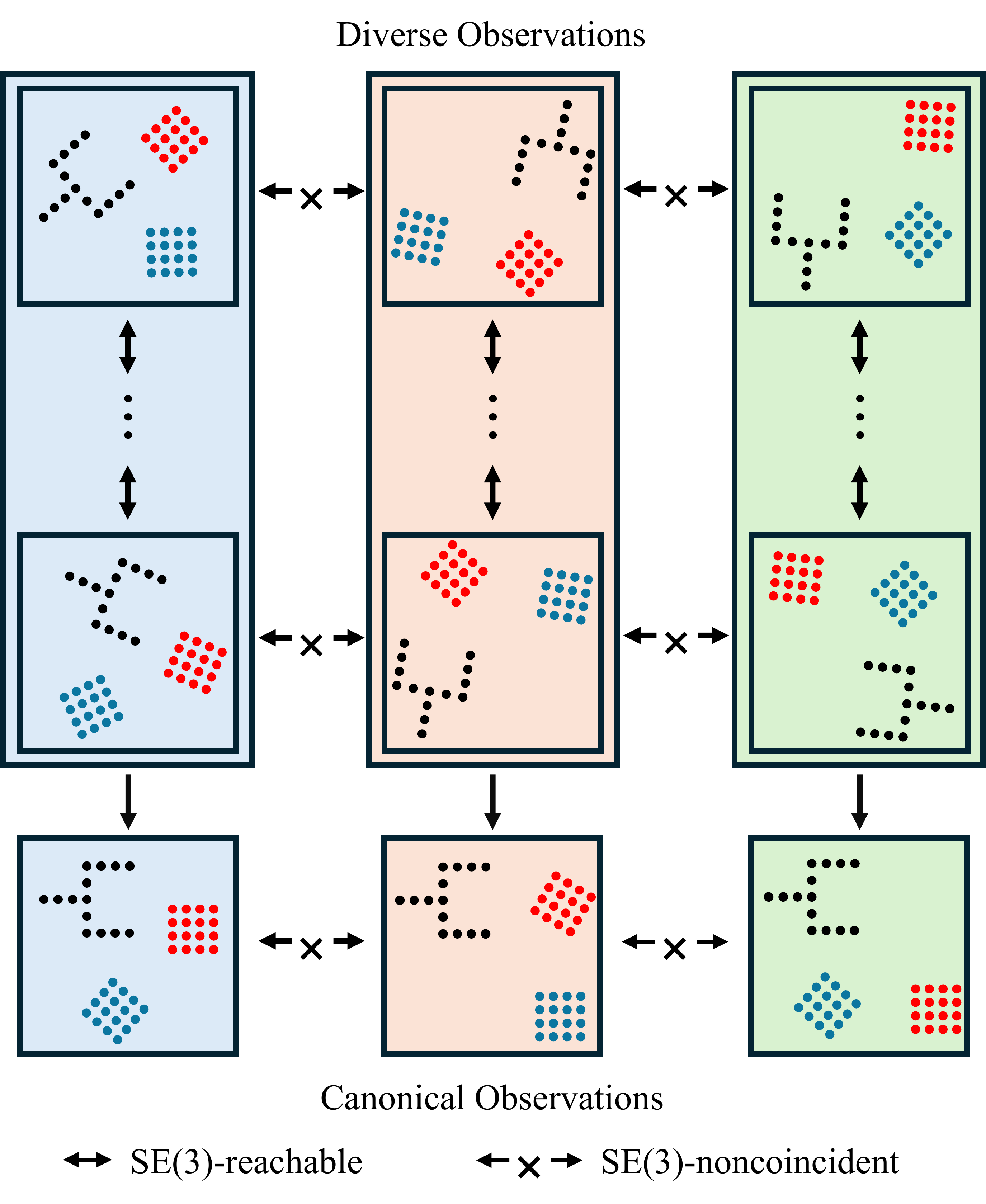}
    \caption{Illustration of the canonicalization process. Within each column, observations differ only by $\mathrm{SE}(3)$ transformations and thus share the same canonical form. Across columns, $\mathrm{SE}(3)$-noncoincident observations cannot be aligned and retain distinct canonical representations.}
    \label{fig:can_obs}
\end{figure}
Let $\Psi$ denote an $\mathrm{SE}(3)$-equivariant network taking point clouds as input. The equivariance property requires that, for any rigid motion $H\in\mathrm{SE}(3)$,
\begin{equation}
\label{eq:equiv_psi}
    \Psi(H X)=H\Psi(X).
\end{equation}
We constrain $\Psi(X)$ to output a rigid transformation $T=\{R,\mathbf{b}\}\in\mathrm{SE}(3)$, where the translation $\mathbf{b}\in\mathbb{R}^{3\times1}$ is computed from the point cloud centroid, and the rotation $R\in\mathrm{SO}(3)$ is constructed by orthonormalizing two $\mathrm{SO}(3)$-equivariant vectors produced by the network. The implementation details follow the Canonical Policy~\cite{zhang2025canonical}.
\begin{proposition}
\label{pro:1}
Let $X$ and $Y$ be two point clouds related by a rigid transformation:
\begin{equation*}
    \exists T_i\in\mathrm{SE}(3), \quad Y=T_i X.
\end{equation*}
We define the canonicalized representations of $X$ and $Y$ as
\begin{equation*}
    \hat{X}=\Psi(X)^{-1}X, \quad \hat{Y}=\Psi(Y)^{-1}Y
\end{equation*}
Then $\hat{X} = \hat{Y}$.
\end{proposition}

\begin{proof}
Let $\Psi(X)=T$, and the canonicalized form of $X$ is given by:
\begin{equation}
\label{eq:can_X}
    \hat{X}=\Psi(X)^{-1}X=T^{-1}X.
\end{equation}
For $Y=T_iX$, the canonicalized form is:
\begin{equation}
\label{eq:can_Y}
    \hat{Y}=\Psi(Y)^{-1}Y=\Psi(T_iX)^{-1}T_iX.
\end{equation}
Using the equivariance property from Equation~\eqref{eq:equiv_psi}, we have:
\begin{equation*}
    \Psi(T_iX)=T_i\Psi(X)=T_iT,
\end{equation*}
and therefore
\begin{equation*}
    \Psi(T_iX)^{-1}=T^{-1}T_i^{-1}.
\end{equation*}
Substituting this back into Equation~\eqref{eq:can_Y}:
\begin{equation*}
    \hat{Y}=\Psi(T_iX)^{-1}T_iX=T^{-1}T_i^{-1}T_iX=T^{-1}X=\hat{X}. 
\end{equation*}
Hence \( \hat{X} = \hat{Y} \), establishing the claim.
\end{proof}

This process is illustrated in Fig.~\ref{fig:can_obs}. The proposition establishes that any two point clouds related by an $\mathrm{SE}(3)$ transformation map to the same canonical representation.
Consequently, by aligning the robot state and action into this canonical frame, we enable pose-invariant imitation learning.

Based on Proposition~\ref{pro:1}, the inference process of the policy can be formulated as:
\begin{equation}
\label{eq:ideal}
    \pi(O)=T\pi(T^{-1}O)=T\pi(\hat{O}),
\end{equation}
where, $T=\Psi(X)$ and $\hat{O}=T^{-1}O$ is the canonicalized observation.
Specifically, the policy $\pi$ predicts a canonical action based on $\hat{O}$, which is subsequently transformed back to the global frame via $T$. This effectively eliminates observation pose variations, allowing the policy to focus on task-relevant structures and generalizing to unseen poses that are rigid transformations of the training data.

\subsection{Noise-Induced Equivariance Deviation}
\begin{figure}[t]
    \centering
    \includegraphics[width=1.0\linewidth]{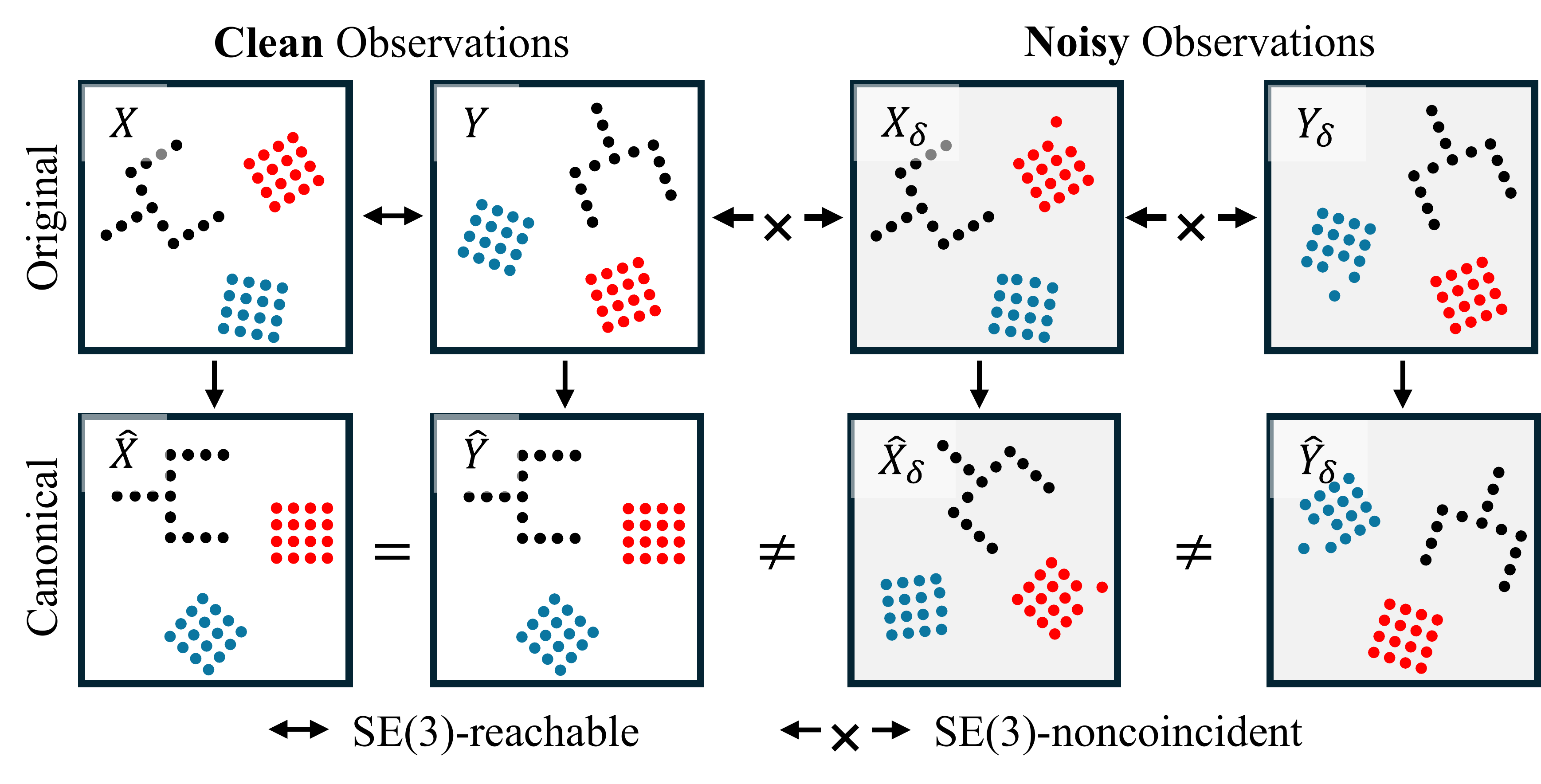}
    \caption{Clean observations related by an $\mathrm{SE}(3)$ transformation canonicalize to the same frame, whereas noise corrupts this relationship: noisy observations become $\mathrm{SE}(3)$-noncoincident and yield inconsistent canonical forms.}
    \label{fig:clean_noise}
\end{figure}
\label{sec:noise_breakdown}
The theoretical guarantees of Proposition~\ref{pro:1} hinge on the idealized assumption that the underlying geometry is consistent, i.e., that $Y$ is an exact rigid transformation of $X$. However, in practice, point cloud observations are inevitably corrupted by sensor noise, partial occlusions, and sampling inconsistencies. As shown in Fig.~\ref{fig:clean_noise}, even minor deviations in a single point can prevent perfect $\mathrm{SE}(3)$ alignment between two nominally identical observations. Consequently, we must analyze the behavior of the canonicalization under noisy inputs.

Let $X_{\delta} = \delta(X)$ denote a noisy point cloud observation, where $\delta$ is a stochastic noise function. Our analysis relies on the following assumption:
\begin{assumption}
We assume that noise primarily corrupts the visual observation $X$, while the robot's proprioceptive state $\mathbf{s}$ remains accurate.
\end{assumption}
Thus, the complete noisy observation is $O_{\delta}=\{X_{\delta}, \mathbf{s}\}$.
To reason about the effect of noise on canonicalization, we first note the following property of the noise function:
\begin{proposition}
\label{pro:2}
The stochastic noise function $\delta$ does not possess the $\mathrm{SE}(3)$-equivariant property. That is, for any rigid transformation $H\in\mathrm{SE}(3)$, the following relationship does not hold:
\begin{equation}
\label{eq:neq}
    \delta(HX)\neq H\delta(X)
\end{equation}
\end{proposition}

\begin{proof}
The function $\delta$ models sensor noise or sampling inconsistency and is typically modeled as an independent stochastic perturbation applied to each point $\mathbf{x} \in X$, such as additive Gaussian noise $\delta(\mathbf{x})=\mathbf{x}+\epsilon$. We examine the commutative relationship between the rigid transformation $H$ and the noise function $\delta$.
If the transformation is applied first, the perturbed point becomes
\begin{equation*}
    \delta(H\mathbf{x})=H\mathbf{x}+\epsilon_1.
\end{equation*}
If noise is applied first, followed by the transformation, we obtain
\begin{equation*}
    H\delta(\mathbf{x})=H(\mathbf{x}+\epsilon_2)=H\mathbf{x}+H\epsilon_2.
\end{equation*}
Since $\epsilon_1$ and $\epsilon_2$ are independently sampled and $H$ is a non-identity rigid motion, in general $H\epsilon_2 \neq \epsilon_1$. Thus, the perturbed points no longer satisfy the rigid-body relationship required for $\mathrm{SE}(3)$ equivariance. Extending to the full point cloud, we obtain $\delta(HX)\neq H\delta(X)$.
\end{proof}

From Proposition~\ref{pro:2}, noise destroys the $\mathrm{SE}(3)$ structure required for equivariance, leading to
\begin{equation}
\label{eq:Psi_noise_break}
\Psi(\delta(HX)) \neq H\Psi(\delta(X)),
\end{equation}
in direct contrast to the ideal equivariance condition in Equation~\eqref{eq:equiv_psi}.
This mismatch directly disrupts canonicalization, preventing observations that differ only by a rigid transformation in the noise-free setting from aligning to a common canonical form under noise. We formalize this breakdown of invariance in the following proposition.

\begin{proposition}
\label{pro:noise_can}
Let $X$ and $Y$ be two noise-free point clouds related by a rigid transformation, i.e., $\exists T_i\in\mathrm{SE}(3)$ such that $Y=T_i X$.
Let the corresponding noisy point clouds be $X_\delta=\delta(X)$ and $Y_\delta=\delta(Y)$. Then, their canonicalized representations are no longer equivalent:
\[
    \hat{X}_\delta \neq \hat{Y}_\delta.
\]
\end{proposition}

\begin{proof}
Their noisy canonical forms are
\begin{align*}
     \hat{X}_\delta&=\Psi(X_\delta)^{-1}X_\delta,\\
    \hat{Y}_\delta&=\Psi(Y_\delta)^{-1}Y_\delta
                =\Psi(\delta(T_i X))^{-1}\delta(T_i X).
\end{align*}
From Proposition~\ref{pro:2}, $\delta(T_iX)\neq T_i\delta(X)=T_iX_\delta$, meaning $X_\delta$ and $Y_\delta$ are no longer related by any rigid motion. Since the equivariance property of $\Psi$ applies only to rigidly related inputs, we must have
\[
    \Psi(\delta(T_iX))\neq T_i\,\Psi(X_\delta).
\]
Substituting into the canonicalization expressions yields
\begin{align*}
    \hat{Y}_\delta     &= \Psi(\delta(T_i X))^{-1}\delta(T_i X) \\     &\neq (T_i \Psi(X_\delta))^{-1} (T_i X_\delta) \\     &= \Psi(X_\delta)^{-1} X_\delta \\     &= \hat{X}_\delta.
\end{align*}
Therefore, the canonicalized representations differ, i.e., \( \hat{X}_\delta \neq \hat{Y}_\delta \), which completes the proof.
\end{proof}

Under noise, two observations that were originally related by a rigid transformation no longer map to a shared canonical frame. Consequently, even when using the canonicalization-based policy in Section~\ref{sec:can_policy_learning}, the learned policy still experiences pose-dependent inconsistencies and fails to generalize reliably to unseen viewpoints related by $\mathrm{SE}(3)$ transformations.

\begin{table}[t]
\renewcommand{\arraystretch}{1.5}
\caption{Key notation used in the paper.}
\label{tab:notation}
\centering
\scriptsize
\begin{tabularx}{\linewidth}{@{}>{\raggedright\arraybackslash}p{0.15\linewidth} X@{}}
\toprule
\textbf{Symbol} & \textbf{Meaning} \\
\midrule
$\pi$ & Policy mapping function that maps observations to actions. \\

$\delta$ & Stochastic noise operator applied to point cloud observations. \\

$T \in \mathrm{SE}(3)$ & Rigid transformation in the $\mathrm{SE}(3)$ group. \\

$\Psi$ & $\mathrm{SE}(3)$-equivariant encoder that maps point clouds to latent representations and predicts canonical transformations. \\

$f$ & Differentiable implicit surface function defining the underlying scene surface
$\mathcal{M}=\{\mathbf{x}\in\mathbb{R}^3 \mid f(\mathbf{x})=0\}$. \\

$X \in \mathbb{R}^{3 \times N}$ & Point cloud observation at a single timestep, consisting of $N$ 3D points. \\

$X_\delta$ & Noisy point cloud observation obtained by applying noise operator $\delta(\cdot)$ to $X$. \\

$X^\star$ & Denoised point cloud produced by the geometric denoising module. \\

$\hat{X}$ & Canonicalized point cloud obtained by removing pose variation via an inverse $\mathrm{SE}(3)$ transformation. \\

$\tilde{X}$ & Augmented point cloud generated by applying stochastic geometric perturbations (e.g., noise, dropout, cropping). \\

$\bar{\mathbf{x}}$ & Local neighborhood mean of a point, computed by averaging coordinates of neighboring points. \\

$\mathbf{s},\;\mathbf{a}$ & Robot proprioceptive state and action, represented in Cartesian end-effector space. \\
\bottomrule
\end{tabularx}
\end{table}

To see where this deviation arises, compare the clean canonical form
\[
    \hat{X}=\Psi(X)^{-1}X
\]
with its noisy counterpart
\[
    \hat{X}_\delta=\Psi(X_\delta)^{-1}X_\delta.
\]
The degradation of the canonicalization process stems from two simultaneous and coupled sources of error: the input point cloud itself is corrupted ($X_\delta \neq X$), and the resultant estimated transformation is unstable ($T_\delta = \Psi(X_\delta) \neq T$). Together, these factors cause the canonical representation to drift away from its ideal value, amplify pose-dependent inconsistencies, and ultimately propagate instability to the final action output through the policy's transformation step (Equation~\eqref{eq:ideal}). This structural failure highlights two fundamental noise robustness challenges that must be resolved to enable reliable imitation learning:

1. How can we obtain a canonical representation $\hat{X}_\delta$ that remains stable despite noisy observations, ensuring feature consistency?

2. How can we guarantee that the estimated transformation $T_\delta$ maintains consistency with the ideal $\mathrm{SE}(3)$ structure required for policy equivariance?

In the subsequent section, we introduce the proposed Noise-Robust
$\mathrm{SE}(3)$ Canonicalization and policy design, which directly address the two challenges identified above and aim to restore invariance and generalization capability.

For clarity, Table~\ref{tab:notation} summarizes the key notation used throughout the paper before we proceed to the Method section.

\section{Method}
In this section, we present our Noise-Robust $\mathrm{SE}(3)$-Equivariant Policy.
Building on the analysis in Section~\ref{sec:noise_breakdown}, we identify two key challenges:
(i) noise corrupts the input geometry, causing the canonical representation to drift and lose feature consistency, and
(ii) the transformation estimated from noisy observations becomes unstable, breaking the $\mathrm{SE}(3)$ structure required for equivariant policy behavior.
To resolve these issues, EquiForm integrates two components: geometric denoising of point clouds and equivariant consistency via contrastive learning. We begin by characterizing the sources and structure of noise in observed point clouds. Following this, we detail the implementation of both the geometric denoising module and the contrastive representation learning method. Finally, we outline the pipeline of the EquiForm policy.

\subsection{Point Cloud Noise Analysis}
\begin{figure}[t]
    \centering
    \includegraphics[width=0.9\linewidth]{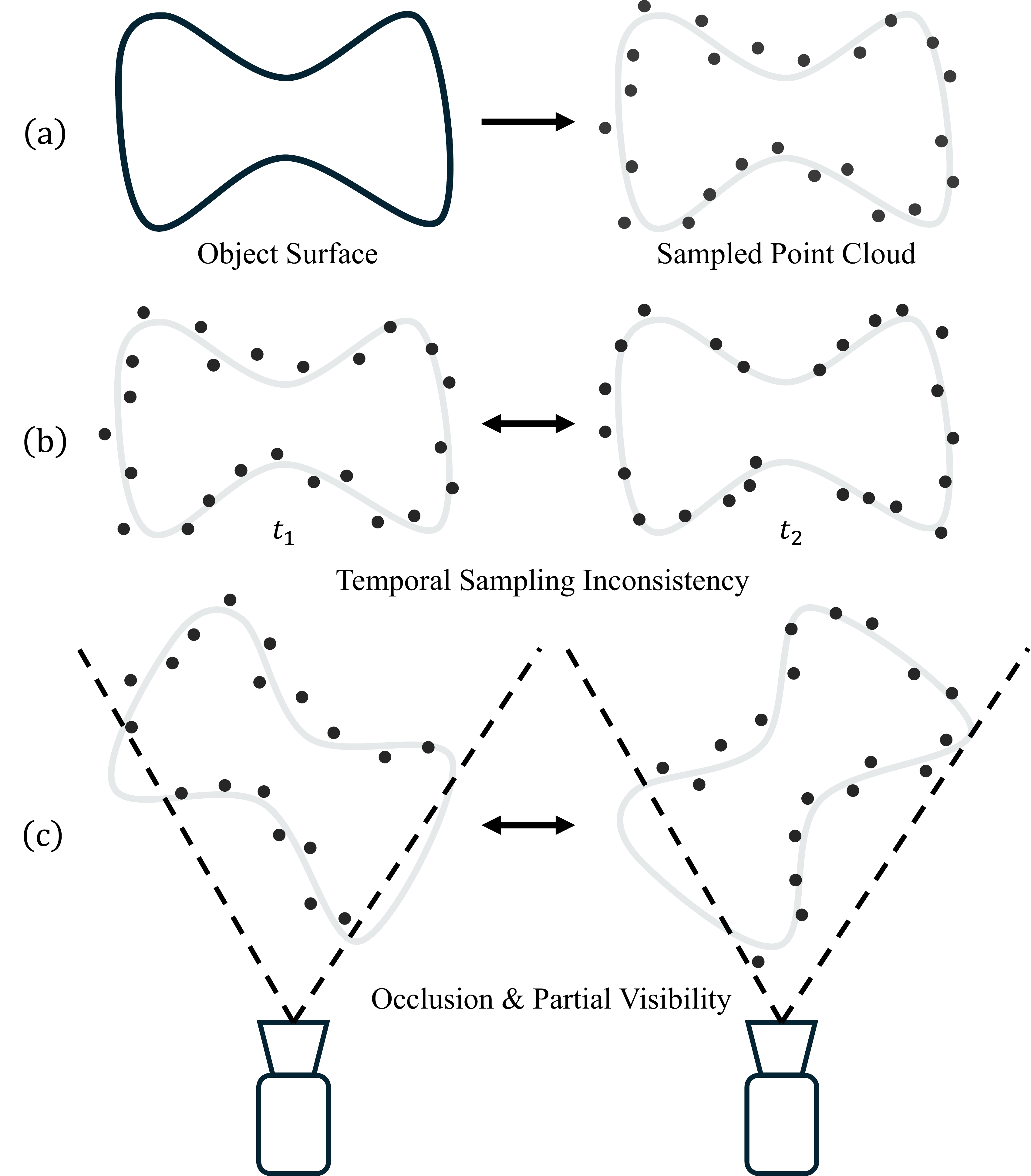}
    \caption{Overview of point cloud inconsistencies that challenge $\mathrm{SE}(3)$-equivariant policy learning. Top: An object’s continuous surface (left) and its ideal discretization into a sampled point cloud (right). Middle: Even without object motion, repeated sampling yields inconsistent point sets due to depth noise, discretization artifacts, and temporal variability. Bottom: Layout changes in the scene cause occlusion and partial visibility, breaking the assumption that observations of the same object remain $\mathrm{SE}(3)$-consistent.}
    \label{fig:noise_overview}
\end{figure}
In vision-conditioned imitation learning, the policy relies heavily on the
observed point clouds to infer the appropriate action.
However, point clouds captured by commodity RGB-D sensors contain several forms of
imperfections that affect both the geometry and the sampling pattern of the
observed scene. Common sources of these inaccuracies include:
\begin{itemize}
    \item \textbf{Depth sensing noise:} small random fluctuations in the depth
          measurement that cause points to deviate from their true surface
          locations.

    \item \textbf{Projection and discretization effects:} converting a continuous
          surface into a pixel-based depth image introduces rounding and
          interpolation artifacts that distort fine geometric details.

    \item \textbf{Sampling inconsistency:} two observations of the same static
      object represent different discrete samples of the underlying surface.
      Downsampling to a fixed-size point set further introduces frame-to-frame
      variation in the selected points.

    \item \textbf{Occlusion and partial visibility:} shifts in viewpoint or
          interactions with the robot and environment can hide portions of the
          object, leading to missing or incomplete observations.
\end{itemize}

As illustrated in Fig.~\ref{fig:noise_overview}(a), a point cloud is only a
discrete approximation of the underlying object surface, and two observations
of the same static scene may not contain exactly the same points. In practice,
depth noise, discretization in the depth image, and sampling inconsistencies
lead to noticeably different point sets across timesteps, as shown in
Fig.~\ref{fig:noise_overview}(b). As a result, two observations that represent
the same physical state may appear different, causing their
learned feature representations to drift apart. Moreover, reliable
$\mathrm{SE}(3)$ equivariance assumes that observations related by a rigid transformation
remain consistent. This assumption is easily violated in practice, as
illustrated in Fig.~\ref{fig:noise_overview}(c): changes in viewpoint or scene
layout alter which parts of the object are visible, introducing occlusions and
partial observations. Such inconsistencies break the $\mathrm{SE}(3)$ relationship between
views, making it difficult for canonicalization modules to infer stable
transformations and ultimately hindering equivariant policy learning.

These disturbances give rise to corrupted observations $X_\delta = \delta(X)$, where points deviate from the underlying surface and the sampling pattern varies across frames. Such perturbations break the geometric consistency required for $\mathrm{SE}(3)$-equivariant canonicalization.
Consequently, our first goal is to recover a geometrically faithful point cloud $X^\star$ that restores the underlying surface structure and yields a stable input for canonicalization. A geometric denoising module is introduced for this purpose.

\subsection{Geometric Denoising of Point Clouds}
\label{sec:geometric_denoising}
\begin{figure}[t]
    \centering
    \includegraphics[width=1.0\linewidth]{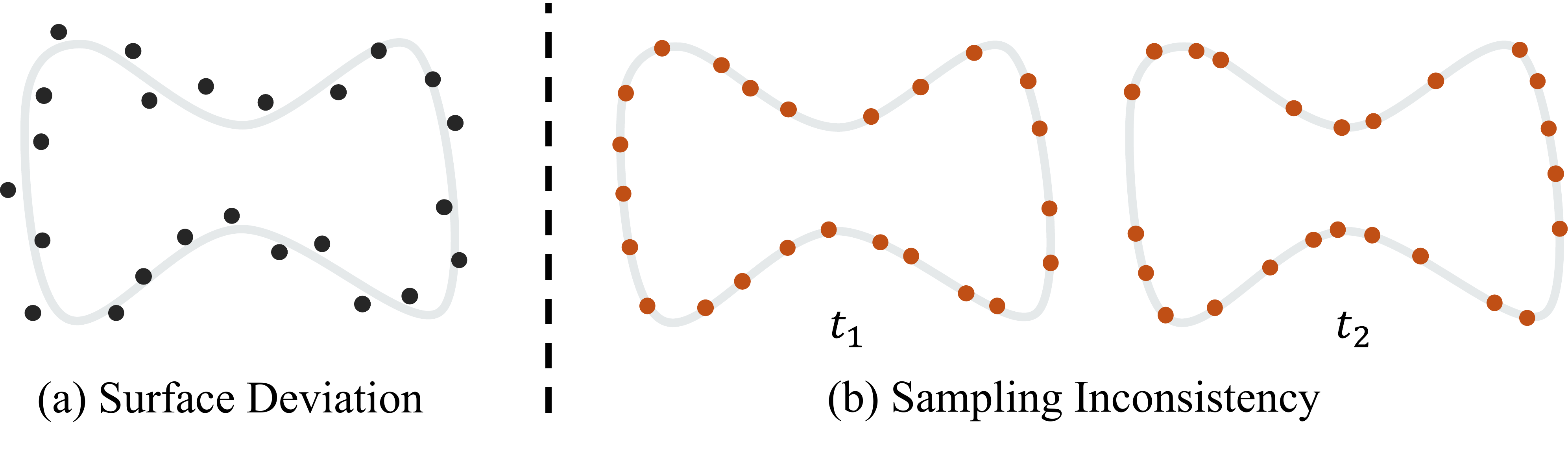}
    \caption{Two common sources of point cloud noise: (a) surface deviation, where individual points deviate from the underlying geometry, and (b) sampling inconsistency, where repeated measurements sample different subsets of the surface.}
    \label{fig:noise_type}
\end{figure}
Formally, we assume that the underlying scene is a smooth surface 
$\mathcal{M}=\{\mathbf{x}\in\mathbb{R}^3 \mid f(\mathbf{x})=0\}$, where 
$f$ is a differentiable implicit function. Given a noisy point set 
$X_\delta=\{\delta(\mathbf{x}_i)\}$, our objective is to recover a cleaned point set 
$X^\star=\{\mathbf{x}^\star_i\}$ that more accurately reflects the geometry of 
$\mathcal{M}$.
We decompose observation noise into two components:
(i) \emph{normal-direction deviation}, in which points drift away from the
underlying surface due to measurement noise, as illustrated in
Fig.~\ref{fig:noise_type}(a); and
(ii) \emph{tangent-direction inconsistency}, arising from sampling variation and
local density irregularities along the surface, as shown in
Fig.~\ref{fig:noise_type}(b).

For notational simplicity, we denote a noisy observation
$\delta(\mathbf{x}_i)$ by $\mathbf{x}_\delta$ in the following derivations,
omitting both the explicit noise operator $\delta(\cdot)$ and the point index
when no ambiguity arises.
To correct each noisy point $\mathbf{x}_\delta$, we apply two corresponding
geometric updates, which are described below.

\paragraph{Normal-direction correction}
Due to measurement noise, an observed point
$\mathbf{x}_\delta \in X_\delta$ generally satisfies $f(\mathbf{x}_\delta) \neq 0$
and deviates from the underlying surface.
We construct a local surface reference as
\begin{equation}
\label{eq:x_mean}
\bar{\mathbf{x}}_\delta
=
\frac{1}{|\mathcal{N}(\mathbf{x}_\delta)|}
\sum_{\mathbf{x}_i \in \mathcal{N}(\mathbf{x}_\delta)}
\mathbf{x}_i,
\end{equation}
where $\mathcal{N}(\mathbf{x}_\delta)$ denotes the local neighborhood of
$\mathbf{x}_\delta$ in the point cloud $X_\delta$.
We approximate $f(\mathbf{x}_\delta)$ using a first-order Taylor expansion
around $\bar{\mathbf{x}}_\delta$:
\begin{equation*}
f(\mathbf{x}_\delta)
\approx
f(\bar{\mathbf{x}}_\delta)
+
\nabla f(\bar{\mathbf{x}}_\delta)^\top
(\mathbf{x}_\delta - \bar{\mathbf{x}}_\delta).
\end{equation*}
Under the assumption of local surface smoothness, $\bar{\mathbf{x}}_\delta$
provides a first-order approximation of a surface point, we have
$f(\bar{\mathbf{x}}_\delta) \approx 0$, yielding
\begin{equation}
\label{eq:f_x_delta}
f(\mathbf{x}_\delta)
\approx
\nabla f(\bar{\mathbf{x}}_\delta)^\top
(\mathbf{x}_\delta - \bar{\mathbf{x}}_\delta).
\end{equation}
Our goal is to compute a corrected point $\mathbf{x}$ that satisfies
$f(\mathbf{x}) = 0$.
Using a first-order approximation of $f(\mathbf{x})$ around
$\mathbf{x}_\delta$, we obtain
\begin{equation*}
f(\mathbf{x})
\approx
f(\mathbf{x}_\delta)
+
\nabla f(\mathbf{x}_\delta)^\top
(\mathbf{x} - \mathbf{x}_\delta)
= 0.
\end{equation*}
Substituting Equation~\eqref{eq:f_x_delta} into the above expression leads to the
linear constraint
\begin{equation*}
\nabla f(\mathbf{x}_\delta)^\top
(\mathbf{x} - \mathbf{x}_\delta)
=
-
\nabla f(\bar{\mathbf{x}}_\delta)^\top
(\mathbf{x}_\delta - \bar{\mathbf{x}}_\delta).
\end{equation*}

In practice, we estimate a local gradient direction for each point from its
discrete neighborhood to reduce the effect of measurement noise.
This neighborhood-based estimation provides a smooth, point-wise approximation
of the local surface orientation.
Under this smoothing operation and the assumption of local surface continuity,
we approximate
$\nabla f(\bar{\mathbf{x}}_\delta) \approx \nabla f(\mathbf{x}_\delta).$
Defining the unit surface normal as
\[
\mathbf{n}_\delta
=
\frac{\nabla f(\mathbf{x}_\delta)}
{\|\nabla f(\mathbf{x}_\delta)\|},
\]
the normal-direction update is then given by
\begin{equation}
\label{eq:normal_update}
\mathbf{x}
=
\mathbf{x}_\delta
-
\mathbf{n}_\delta \mathbf{n}_\delta^\top
(\mathbf{x}_\delta - \bar{\mathbf{x}}_\delta),
\end{equation}
which projects the noisy observation back toward the implicit surface by
removing its normal-direction deviation.

\begin{figure}[t]
    \centering
    \includegraphics[width=0.8\linewidth]{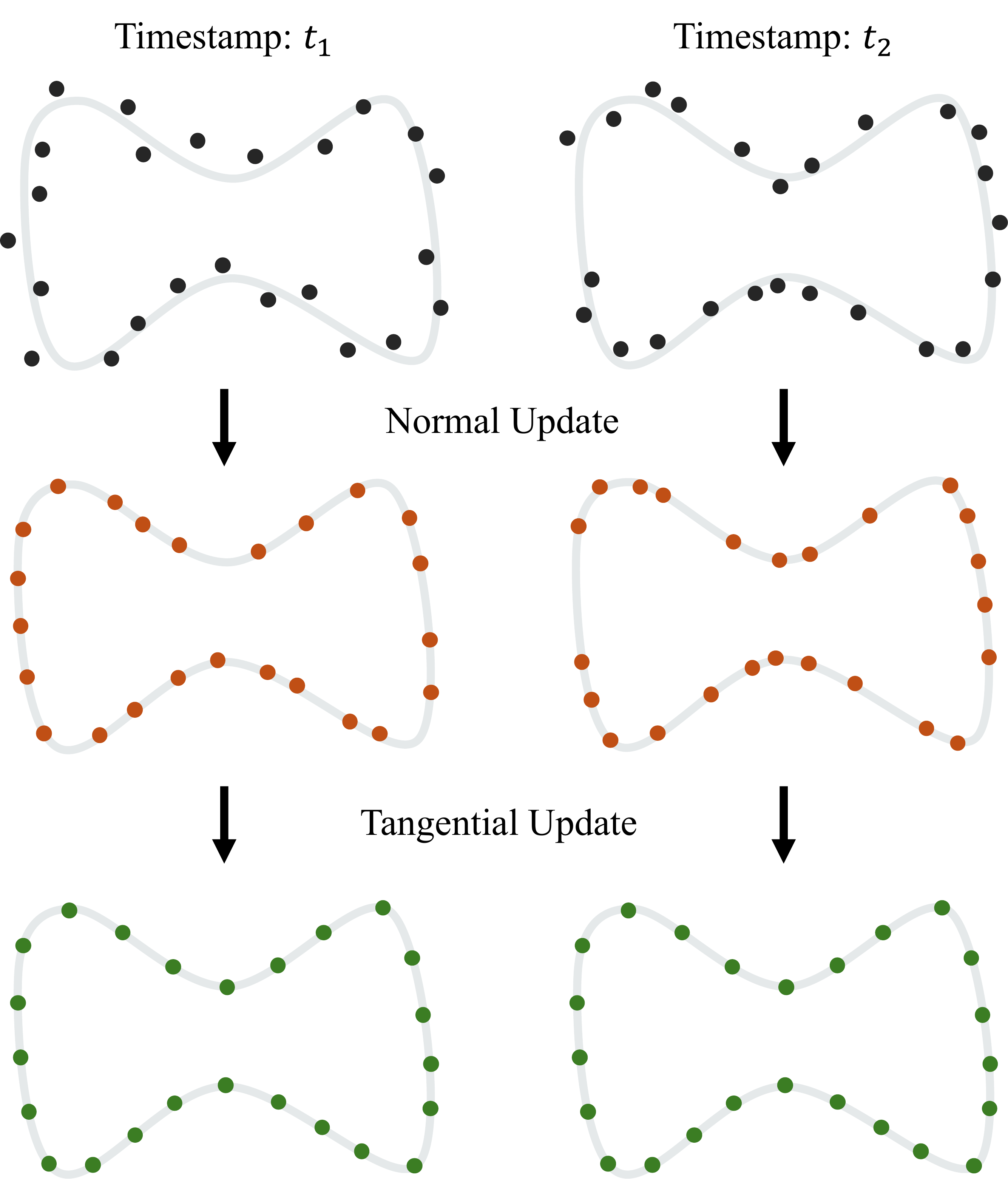}
    \caption{Point clouds captured at different timestamps exhibit geometric noise and sampling inconsistency. A normal update first reduces surface deviation, followed by a tangential update that enforces sampling consistency, yielding a stable and smooth representation.}
    \label{fig:denoise}
\end{figure}
\paragraph{Tangent-direction correction}
After removing off-surface deviation via the normal-direction update, the
corrected point $\mathbf{x}$ lies approximately on the underlying surface.
However, due to sampling variation and local density irregularity, the point
cloud may still exhibit inconsistencies along the tangent directions of the
surface.
Such tangent-direction noise does not violate the surface constraint
$f(\mathbf{x}) = 0$, but leads to non-uniform sampling across observations.

To mitigate these effects, we introduce a tangent-direction refinement that
enforces local sampling consistency while preserving the surface constraint.
Specifically, we aim to suppress high-frequency sampling variation along the
surface by applying a spatial low-pass filtering operation within the tangent
plane.
This is achieved by encouraging the correction displacement to align with the
local neighborhood offset $\mathbf{x}-\bar{\mathbf{x}}$, while restricting the
correction to lie entirely in the tangent space.

Formally, we seek an updated point
$\mathbf{x}^\star = \mathbf{x} - \Delta \mathbf{x}$,
where $\Delta \mathbf{x}$ denotes the tangent-direction correction displacement.
We consider the following constrained optimization problem:
\begin{equation}
\label{eq:tangent_opt}
\min_{\Delta \mathbf{x}}
\;
\|\Delta \mathbf{x} - (\mathbf{x} - \bar{\mathbf{x}})\|_2^2
\qquad
\text{s.t.} \quad
\mathbf{n}^\top \Delta \mathbf{x} = 0 ,
\end{equation}
where the constraint enforces that the correction lies entirely within the
tangent plane defined by the local surface orientation.
The Lagrangian associated with Equation~\eqref{eq:tangent_opt} is given by
\begin{equation*}
\mathcal{L}(\Delta \mathbf{x}, \lambda)
=
\|\Delta \mathbf{x} - (\mathbf{x} - \bar{\mathbf{x}})\|_2^2
+
\lambda\, \mathbf{n}^\top \Delta \mathbf{x}.
\end{equation*}
Taking the derivative with respect to $\Delta \mathbf{x}$ and setting it to zero
yields
\begin{equation}
\label{eq:tangent_dx}
\Delta \mathbf{x}
=
(\mathbf{x} - \bar{\mathbf{x}})
-
\frac{\lambda}{2}\,\mathbf{n}.
\end{equation}
Enforcing the tangency constraint in Equation.~\eqref{eq:tangent_opt} leads to
\begin{equation*}
\lambda
=
2\,\mathbf{n}^\top
(\mathbf{x} - \bar{\mathbf{x}}).
\end{equation*}
Substituting this result into Equation.~\eqref{eq:tangent_dx}, we obtain
\begin{equation*}
\Delta \mathbf{x}
=
\left(
I
-
\mathbf{n} \mathbf{n}^\top
\right)
(\mathbf{x} - \bar{\mathbf{x}}).
\end{equation*}
Finally, the tangent-direction update is given by
\begin{equation}
\label{eq:tangent_update}
\mathbf{x}^\star
=
\mathbf{x}
-
\left(
I
-
\mathbf{n} \mathbf{n}^\top
\right)
(\mathbf{x} - \bar{\mathbf{x}}),
\end{equation}
which performs spatial smoothing within the tangent plane while preserving the
surface constraint.

\begin{algorithm}[t]
\caption{Progressive Geometric Denoising}
\label{alg:geometric_denoising}
\KwIn{Noisy point cloud $X_\delta = \{\mathbf{x}_\delta^i\}_{i=1}^N$, neighborhood size $k$}
\KwOut{Denoised point cloud $X^\star$}
\quad// Normal-direction correction \;
1:  \textbf{foreach} $\mathbf{x}_\delta \in X_\delta$ \textbf{do}\;
2:\quad Compute $k$-NN index set $\mathcal{I}(\mathbf{x}_\delta)$ in $X_\delta$\;
3:\quad Estimate local surface normals
      $\{\mathbf{n}_j\}_{j \in \mathcal{I}(\mathbf{x}_\delta)}$
      via PCA on centered neighbors\;
4:\quad Compute local mean
      $\bar{\mathbf{x}}_\delta = \frac{1}{k} \sum_{j \in \mathcal{I}(\mathbf{x}_\delta)} \mathbf{x}_j$
      (Equation~\eqref{eq:x_mean})\;
5:\quad Compute aggregated normal
      $\mathbf{n}_\delta =
      \mathrm{Normalize}\!\left(
      \frac{1}{k} \sum_{j \in \mathcal{I}(\mathbf{x}_\delta)} \mathbf{n}_j
      \right)$\;
6:\quad Update point by normal projection
      $\mathbf{x}
      =
      \mathbf{x}_\delta
      -
      \mathbf{n}_\delta \mathbf{n}_\delta^\top
      (\mathbf{x}_\delta - \bar{\mathbf{x}}_\delta)$
      (Equation~\eqref{eq:normal_update})\;
7:\quad Add $\mathbf{x}$ to $X$\;
8: \textbf{end}\;
\quad// Tangent-direction correction \;
9: \textbf{foreach} $\mathbf{x} \in X$ \textbf{do}\;
10:\quad Compute $k$-NN index set $\mathcal{I}(\mathbf{x})$ in $X$\;
11:\quad Estimate local normals and compute
       $\bar{\mathbf{x}}$ and $\mathbf{n}$
       as in the normal-direction correction\;
12:\quad Update point by tangent projection
       $\mathbf{x}^\star
       =
       \mathbf{x}
       -
       \left(I - \mathbf{n}\mathbf{n}^\top\right)
       (\mathbf{x} - \bar{\mathbf{x}})$
       (Equation~\eqref{eq:tangent_update})\;
13:\quad Add $\mathbf{x}^\star$ to $X^\star$\;
14: \textbf{end}\;

15: \Return $X^\star$\;
\end{algorithm}

Together, Equations~\eqref{eq:normal_update} and~\eqref{eq:tangent_update}
define a progressive geometric denoising pipeline, which is illustrated in
Fig.~\ref{fig:denoise} and formally described in Algorithm~\ref{alg:geometric_denoising}.
The normal-direction update first enforces surface consistency by removing
off-surface deviation, followed by a tangent-direction refinement that improves
sampling uniformity along the surface.
The resulting point cloud exhibits improved geometric stability across
observations, providing a reliable input for subsequent canonicalization and
$\mathrm{SE}(3)$-equivariant policy learning.

\subsection{Equivariant Consistency via Contrastive Learning}
\label{sec:contrastive}
\begin{figure}[t]
    \centering
    \includegraphics[width=0.85\linewidth]{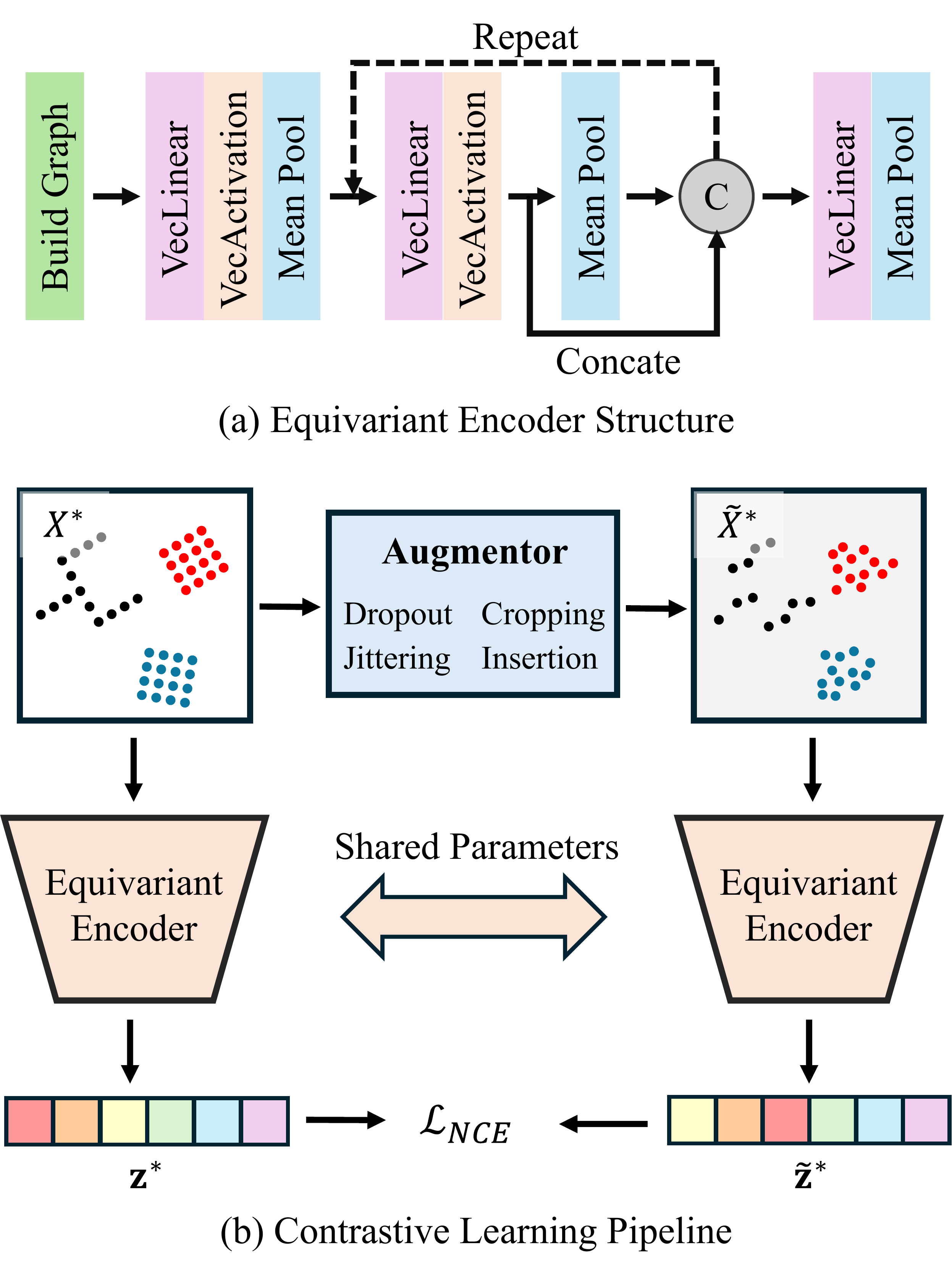}
    \caption{Overview of the equivariant contrastive learning framework.
(a) Equivariant encoder structure.
The encoder is constructed using Vector Neuron (VN) layers
to extract $\mathrm{SE}(3)$-equivariant features from the input point clouds.
(b) Contrastive learning pipeline.
Starting from a denoised point cloud $X^\star$, an augmented view
$\tilde{X}^\star$ is generated via geometric perturbations, including
jittering, cropping, and point dropout or insertion.
Both views are processed by a shared equivariant encoder, and a contrastive loss is applied in the equivariant feature space to encourage feature consistency under noise, while preserving the $\mathrm{SE}(3)$-equivariant structure of the representation.
}
    \label{fig:contrastive}
\end{figure}
So far, Question~1 in Section~\ref{sec:noise_breakdown}, illustrated in
Fig.~\ref{fig:noise_overview}(b), has been addressed: after geometric denoising,
the point cloud $X^\star$ becomes stable and reliable, with surface-deviation
noise reduced and sampling rendered more uniform.
However, as illustrated in Fig.~\ref{fig:noise_overview}(c), even when observing
the same object across different frames, variations in scene layout or camera
viewpoint may introduce occlusions and partial visibility.
As a result, the observed point clouds can still exhibit noticeable
inconsistencies, leaving Question~2 unresolved and posing challenges for the
subsequent canonicalization process.
In this subsection, we introduce a contrastive learning strategy to alleviate
this issue and promote equivariant consistency across observations.

Contrastive learning has emerged as a widely adopted paradigm for self-supervised
representation learning, with the objective of learning an embedding space in
which semantically similar samples are mapped close together, while dissimilar
samples are pushed apart.
This is typically achieved by constructing positive pairs that share the same
underlying content under different views or perturbations, along with negative
pairs drawn from distinct samples.
By enforcing consistency across augmented views, contrastive learning has been
shown to improve robustness to nuisance factors such as noise, partial
observations, and viewpoint variations.

Formally, let $\mathbf{x}_i$ denote an input sample, and let
$\tilde{\mathbf{x}}_i^{(1)}$ and $\tilde{\mathbf{x}}_i^{(2)}$ be two augmented
views generated from the same underlying sample.
An encoder $\Phi(\cdot)$ maps each view to a latent representation
$\tilde{\mathbf{z}}_i^{(1)} = \Phi(\tilde{\mathbf{x}}_i^{(1)})$ and
$\tilde{\mathbf{z}}_i^{(2)} = \Phi(\tilde{\mathbf{x}}_i^{(2)})$.
Contrastive learning encourages the representations of such positive pairs to be
similar, while separating them from representations of other samples in the
batch, which serve as negative examples.

We adopt the InfoNCE loss~\cite{oord2018representation}, which is commonly
used in contrastive learning, to enforce this behavior.
For a positive pair
$(\tilde{\mathbf{z}}_i^{(1)}, \tilde{\mathbf{z}}_i^{(2)})$, the loss is defined as
\begin{equation}
\label{eq:infoNCE}
\mathcal{L}_{\mathrm{NCE}}
=
- \log
\frac{\exp\left( \mathrm{sim}(\tilde{\mathbf{z}}_i^{(1)}, \tilde{\mathbf{z}}_i^{(2)}) / \tau \right)}
{\sum_{j=1}^{N}
\exp\left( \mathrm{sim}(\tilde{\mathbf{z}}_i^{(1)}, \tilde{\mathbf{z}}_j^{(2)}) / \tau \right)},
\end{equation}
\begin{figure*}[t]
    \centering
    \includegraphics[width=0.95\linewidth]{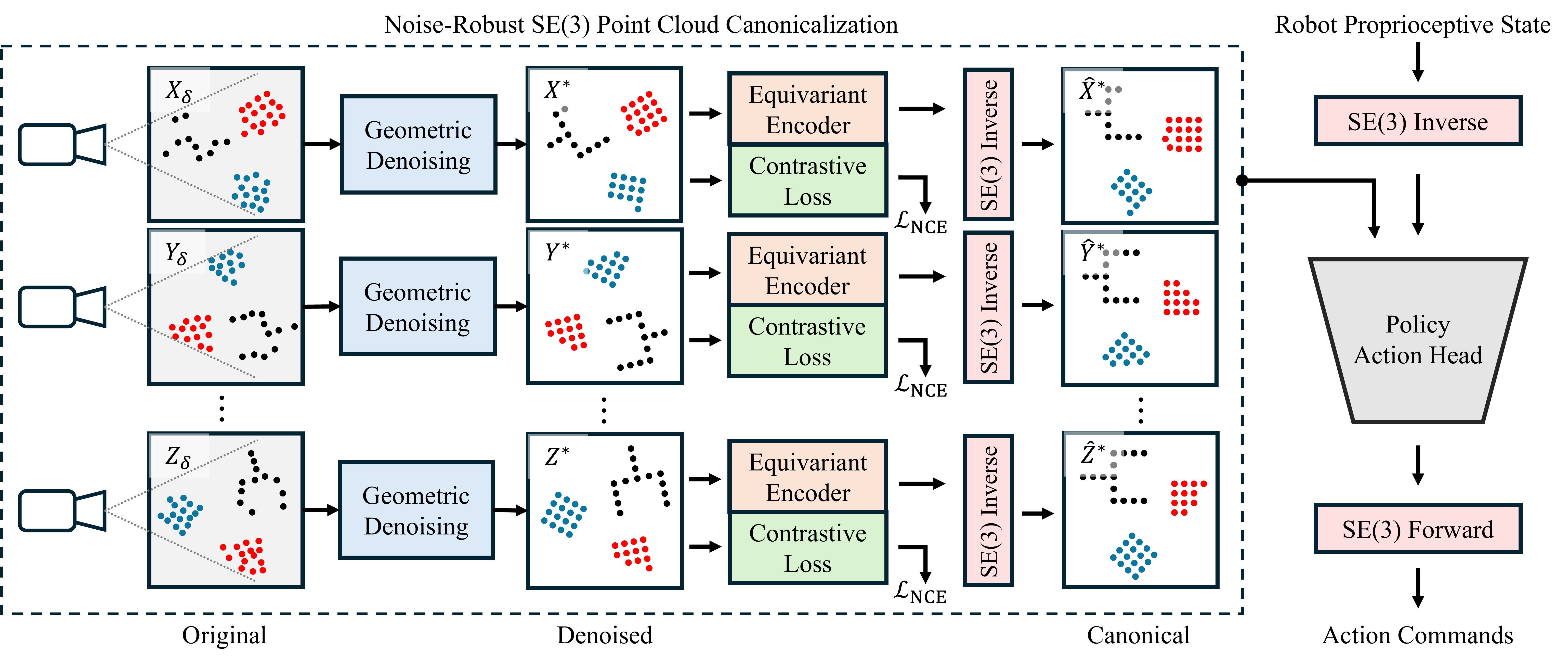}
    \caption{
Overview of the proposed EquiForm policy learning pipeline.
Given noisy point cloud observations $X_\delta, Y_\delta, \ldots, Z_\delta$,
a geometric denoising module is first applied to suppress surface deviation and
sampling inconsistencies, yielding denoised point clouds
$X^\star, Y^\star, \ldots, Z^\star$.
The denoised observations are then processed by a shared
$\mathrm{SE}(3)$-equivariant encoder to extract equivariant feature
representations.
During training, a contrastive loss (InfoNCE) is imposed on the equivariant
features obtained from augmented views of the same underlying point cloud,
encouraging feature consistency under geometric perturbations and partial
observations.
Based on the stabilized equivariant features, the $\mathrm{SE}(3)$-equivariant
encoder predicts an input-specific rigid transformation, which is then inverted
to obtain a canonicalized point cloud representation.
The canonical representation, together with robot proprioceptive inputs, is fed
into the action head to predict control commands, which are finally mapped back
to the original frame via the corresponding $\mathrm{SE}(3)$ forward
transformation.
    }
    \label{fig:pipeline}
\end{figure*}
where $\mathrm{sim}(\cdot,\cdot)$ denotes cosine similarity and $\tau$ is a
temperature parameter.
Minimizing this loss increases the similarity between representations of
positive pairs while reducing their similarity to negative samples.

In our setting, the input sample corresponds to a denoised point cloud
$X^\star$, and the encoder is an equivariant network
$\Psi(\cdot)$, whose architecture is illustrated in
Fig.~\ref{fig:contrastive}(a).
A natural instantiation of such an equivariant encoder is the Vector Neuron
(VN) framework~\cite{VN}, which directly takes point cloud coordinates as input
and operates on vector-valued features rather than scalar representations.

An augmented view $\tilde{X}^\star$ is generated from $X^\star$ via stochastic
geometric perturbations, including Gaussian noise injection as well as
variations in viewpoint or point sampling, such as cropping, point dropout,
and point insertion, as illustrated in Fig.~\ref{fig:contrastive}(b).
The denoised point cloud and its augmented counterpart are processed by the
same equivariant encoder $\Psi(\cdot)$ with shared parameters, producing hidden
equivariant feature representations $\mathbf{z}^\star$ and
$\tilde{\mathbf{z}}^\star$, respectively.
By applying the contrastive objective in this equivariant feature space, we
encourage the encoder to produce consistent representations under geometric
perturbations and partial observations, thereby improving equivariant
consistency without breaking the desired $\mathrm{SE}(3)$ structure required
for subsequent canonicalization and policy learning.

As discussed in Section~\ref{sec:can_policy_learning}, the canonicalization
procedure maps an input point cloud to a rigid transformation.
When the hidden features of two input point clouds are similar, the resulting
estimated rigid transformations are also expected to be close.
By explicitly regularizing the parameters of the equivariant encoder through
contrastive learning, we encourage robustness to noise and partial
observations at the feature level.
Consequently, in addressing Question~2 in Section~\ref{sec:noise_breakdown},
the estimated transformation $T$ remains consistent between a clean point
cloud and its noisy or partially observed counterparts.

\subsection{EquiForm Policy Learning Pipeline}
We now present the complete EquiForm policy learning pipeline, which integrates
geometric denoising and equivariant contrastive representation learning into a
unified framework for noise-robust $\mathrm{SE}(3)$-equivariant policy learning.
An overview of the pipeline is illustrated in Fig.~\ref{fig:pipeline}.

At each timestep, the policy receives a raw point cloud observation
$X_\delta = \delta(X)$ captured by an RGB-D sensor.
As analyzed in Section~\ref{sec:noise_breakdown}, such observations are corrupted
by surface deviation, sampling inconsistency, and partial visibility, which
violate the geometric assumptions required for stable canonicalization.
To address this, the observation is first processed by the geometric denoising
module introduced in Section~\ref{sec:geometric_denoising}.
This module applies a normal-direction correction to suppress off-surface
deviation, followed by a tangent-direction refinement that enforces local
sampling consistency.
The output is a geometrically faithful and temporally stable point cloud
$X^\star$, which provides a reliable approximation of the underlying object
surface.

The denoised point cloud $X^\star$ is then encoded by an
$\mathrm{SE}(3)$-equivariant encoder $\Psi(\cdot)$ to extract hidden equivariant
feature representations.
While geometric denoising restores surface-level consistency, it does not fully
resolve ambiguities introduced by occlusion and partial observations.
To further stabilize the representation, contrastive learning is employed
during training to regularize the equivariant feature space, as described
in Section~\ref{sec:contrastive}.
Specifically, stochastic geometric perturbations are applied to $X^\star$ to
generate augmented views that simulate viewpoint changes, sampling variation,
and noise.
Both the original and augmented point clouds are processed by the same
equivariant encoder with shared parameters, and a contrastive loss (InfoNCE) is
imposed on the resulting equivariant features.
This regularization encourages feature-level consistency across perturbations
while preserving the underlying $\mathrm{SE}(3)$ structure, thereby mitigating
representation drift caused by noise and occlusion.

\begin{figure*}[t]
    \centering
    \includegraphics[width=1.0\linewidth]{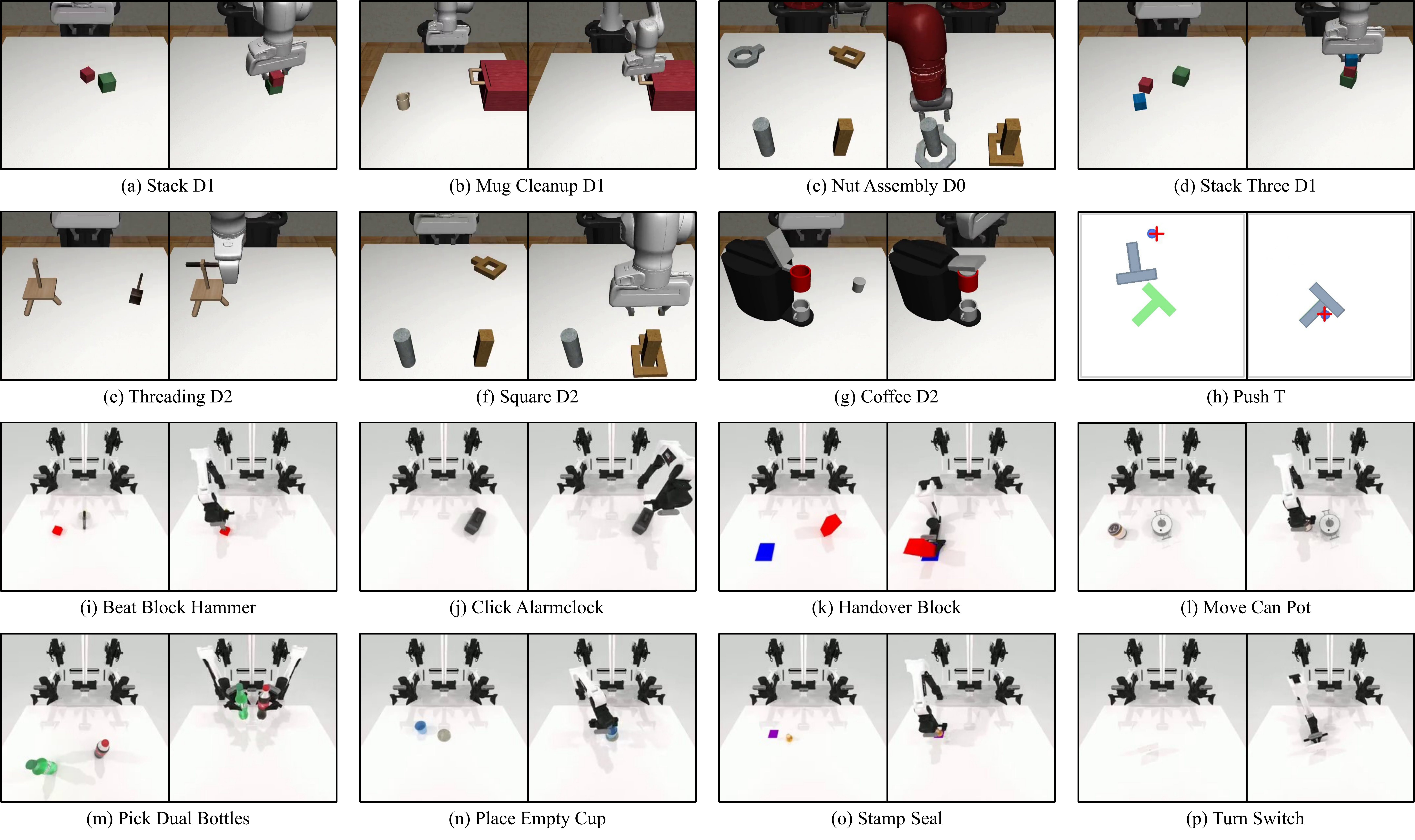}
    \caption{Visualization of the 16 simulated manipulation tasks used in the simulation benchmark.
    Subfigures (a)–(g) are from MimicGen~\cite{robomimic}, (h) is from Diffusion Policy~\cite{diffusion_policy}, and (i)–(p) are from RoboTwin~\cite{chen2025robotwin}.
    In each subfigure, the left image shows the initial state, and the right image shows the goal state of the task.}
    \label{fig:sim_tasks}
\end{figure*}
\begin{table*}[t]
\caption{Task success rates (\%) across different policies and tasks.}
\label{tab:sim_comparison}
\centering
\renewcommand{\arraystretch}{1.15}
\setlength{\tabcolsep}{5pt}

\newcolumntype{C}{>{\centering\arraybackslash}X}
\newcommand{\tblwidth}{0.97\textwidth}

\begin{tabularx}{\tblwidth}{lCCCCCC}
\toprule

 & \textbf{Stack D1} & \textbf{Mug Cleanup D1} & \textbf{Nut Assembly D0} & \textbf{Stack Three D1} & \textbf{Threading D2} & \textbf{Square D2} \\
\midrule
DP3 & 23 & 28 & 21 & 1 & 8 & 3 \\
Canonical Policy & 79 & 38 & 42 & 9 & 15 & 15 \\
\rowcolor{gray!30}
EquiForm & \textbf{94} & \textbf{50} & \textbf{72} & \textbf{24} & \textbf{39} & \textbf{37} \\

\midrule

 & \textbf{Coffee D2} & \textbf{PushT} & \textbf{Beat Block Ham.} & \textbf{Click Alarmclock} & \textbf{Handover Block} & \textbf{Move Can Pot} \\
\midrule
DP3 & 43 & 71 & 72 & 77 & 70 & 70 \\
Canonical Policy & 51 & 89 & \textbf{93} & 65 & 68 & 84 \\
\rowcolor{gray!30}
EquiForm & \textbf{52} & \textbf{93} & 87 & \textbf{89} & \textbf{85} & \textbf{90} \\

\midrule

 & \textbf{Pick Dual Bottles} & \textbf{Place Empty Cup} & \textbf{Stamp Seal} & \textbf{Turn Switch} & \multicolumn{2}{c}{\textbf{Average Success Rate}} \\
\midrule
DP3 & 60 & 65 & 18 & 46 & \multicolumn{2}{c}{42.3} \\
Canonical Policy & \textbf{96} & \textbf{81} & 34 & 46 & \multicolumn{2}{c}{56.6} \\
\rowcolor{gray!30}
EquiForm & 86 & 76 & \textbf{43} & \textbf{49} & \multicolumn{2}{c}{\textbf{66.6}} \\

\bottomrule
\end{tabularx}

\vspace{0.4em}
\footnotesize
\parbox{\tblwidth}{
We report the mean success rate across 16 simulation benchmarks.
\textbf{Bold} indicates the best performance for each task.
All policies are evaluated using point cloud observations containing only 3D coordinates, without incorporating color information.
}
\end{table*}

Based on the stabilized equivariant features, the
$\mathrm{SE}(3)$-equivariant encoder further predicts a rigid transformation
$T \in \mathrm{SE}(3)$ for each input observation, which serves as the
canonicalization transform.
Specifically, the encoder first maps the input point cloud to an intermediate
equivariant feature representation, and subsequently regresses the
corresponding rigid transformation from this feature.
While the canonicalization transform is obtained from the encoder output, the
contrastive loss is applied at the level of the intermediate equivariant
features, encouraging feature consistency across noisy and partially observed
inputs.
This feature-level regularization stabilizes the downstream transformation
prediction, resulting in consistent estimates of $T$ under geometric
perturbations.
The canonicalized point cloud is then obtained by applying the inverse
transformation $T^{-1}$, aligning the input to a shared canonical frame that is
independent of viewpoint and observation-specific artifacts.

Finally, the canonicalized representation is provided as input to the policy
network, together with robot proprioceptive information, to predict control
commands.
These actions are mapped back to the original observation frame via the
corresponding $\mathrm{SE}(3)$ forward transformation.
Because the policy operates on canonicalized inputs and the transformations
preserve $\mathrm{SE}(3)$ structure, the resulting policy is equivariant by
construction and exhibits improved robustness to sensor noise, sampling
variation, and partial observability at test time.

\section{Simulation Experiments}
We evaluate the proposed EquiForm policy through a comprehensive set of simulation experiments that assess both task performance and robustness.
We begin by describing the simulation environment, task setup, baseline methods, and training details.
We then benchmark EquiForm against representative point cloud policy baselines to evaluate overall task performance.
To further analyze robustness, we examine policy behavior under controlled noise perturbations as well as under explicit rigid transformations applied to the input observations.
Finally, we conduct ablation studies to quantify the individual contributions of the geometric denoising module and the equivariant contrastive learning component.

\subsection{Experimental Setup}
We evaluate the proposed EquiForm policy on a diverse set of simulated
manipulation tasks, as illustrated in Fig.~\ref{fig:sim_tasks}.
The benchmark consists of three categories of tasks that vary in manipulation skills, object geometry, and interaction complexity.

\textbf{MimicGen Tasks}~\cite{robomimic}, shown in
Fig.~\ref{fig:sim_tasks}(a)--(g), consist of a variety of single-arm household manipulation scenarios, such as coffee making and table organization. Point cloud observations are captured from four fixed camera views. The original dataset provides point clouds with 1024 points per observation.

\textbf{Push-T Task}, shown in Fig.~\ref{fig:sim_tasks}(h), is adapted from Diffusion Policy~\cite{diffusion_policy}.
Originally formulated as a 2D task, we extend it to 3D by lifting each of the nine 2D keypoints into $\mathbb{R}^3$ with a zero $z$-coordinate, forming a sparse 3D point cloud observation.
The objective is to push a T-shaped block to align it with a target pose.

\textbf{RoboTwin Tasks}~\cite{chen2025robotwin}, visualized in
Fig.~\ref{fig:sim_tasks}(i)--(p), are bimanual manipulation tasks designed to evaluate robustness and generalization under complex geometric interactions.
RoboTwin provides a large-scale dual-arm benchmark with diverse object
geometries and structured domain randomization, resulting in varied geometric configurations in point cloud observations.
Point clouds are captured from both head-mounted and wrist-mounted camera views.
Similar to the MimicGen tasks, the RoboTwin dataset provides point cloud observations with 1024 points per instance.

To assess both task performance and robustness, we compare EquiForm against representative point cloud–based policies under consistent experimental settings.

\textbf{DP3}~\cite{DP3} is a diffusion-based policy that directly processes raw point cloud observations to actions using a simple MLP-based point cloud encoder.

\textbf{Canonical Policy}~\cite{zhang2025canonical} employs an
$\mathrm{SE}(3)$-equivariant encoder based on the Vector Neuron framework~\cite{VN} to estimate a canonical transformation, followed by policy learning in the canonical space.
EquiForm builds upon this framework by introducing geometric denoising and equivariant contrastive regularization to improve robustness under noisy and partially observed inputs.

Since EquiForm applies rigid transformations in $\mathrm{SE}(3)$ to both
proprioceptive states and actions, we adopt Cartesian-space end-effector representations rather than joint-space representations~\cite{zhao2024aloha}.
In addition, actions are parameterized using absolute end-effector poses rather than relative pose increments, following the canonical policy formulation~\cite{zhang2025canonical}.
This choice ensures that the $\mathrm{SE}(3)$ inverse and forward transformations are well defined and physically meaningful.

EquiForm operates as a preprocessing framework on input observations, mapping $\mathrm{SE}(3)$-related point cloud observations into a shared canonical space.
As a result, the subsequent policy action head is not tied to a specific implementation and can be flexibly instantiated using different generative or regression-based models, such as diffusion models or flow matching.
In this work, we adopt the diffusion model used in Diffusion Policy to predict actions from canonicalized observations, ensuring a fair comparison with DP3 and Canonical Policy.

The implementation details of EquiForm follow the settings described in
Sections~\ref{sec:geometric_denoising} and~\ref{sec:contrastive}.
In the geometric denoising module, the neighborhood size used to compute local mean estimates is set to 16 points.
For equivariant contrastive learning, point cloud coordinates are normalized to the range $[-1, 1]$, and stochastic geometric augmentations are applied, including Gaussian noise sampled from $\mathcal{N}(0, 0.1)$, as well as random point dropout, point insertion, and cropping, each affecting 10\% of the points.
The temperature parameter $\tau$ in Equation~\eqref{eq:infoNCE} is set to 0.1.

Furthermore, to reduce training time and memory consumption, all point clouds are downsampled to 256 points using farthest point sampling (FPS)~\cite{qi2017pointnet}. This reduction does not affect final task performance, consistent with prior
findings~\cite{zhang2025canonical}.
All remaining network architectures and training hyperparameters follow the settings used in the Canonical Policy implementation.
Finally, across all tasks, point cloud observations consist of spatial
coordinates only, without any color information.
This design choice follows prior work~\cite{DP3,zhang2025canonical} and promotes generalization across variations in object appearance and visual texture.

\subsection{Benchmarking with Point Cloud Policies}
\begin{figure*}[t]
    \centering
    \includegraphics[width=0.9\linewidth]{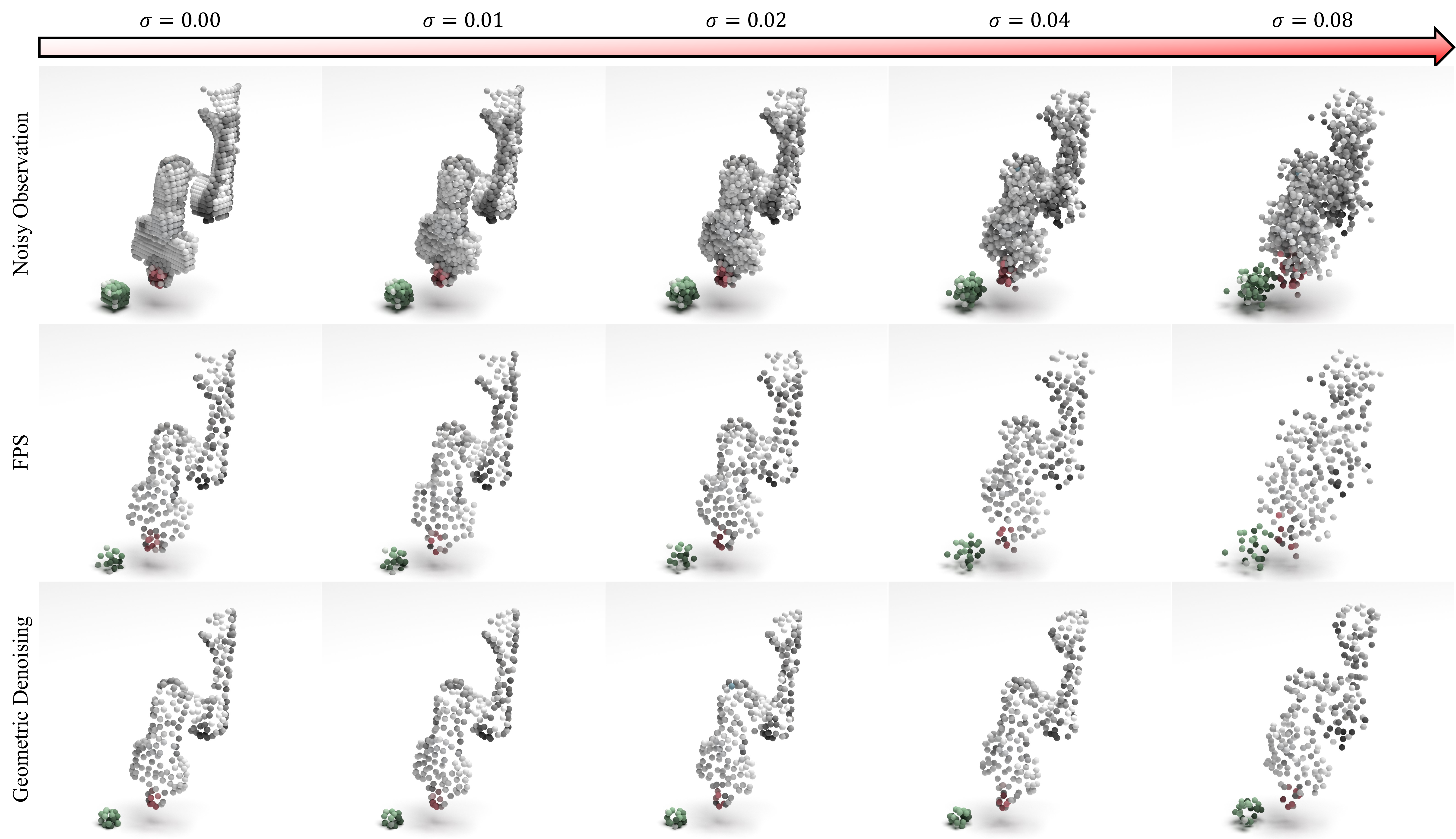}
    \caption{Qualitative visualization of geometric denoising under increasing Gaussian noise. Point cloud observations are corrupted with isotropic Gaussian noise of increasing standard deviation $\sigma$ (left to right). We compare the noisy input, farthest point sampling (FPS), and the proposed geometric denoising. Geometric denoising preserves surface structure and spatial consistency under severe noise, whereas FPS alone fails to recover coherent geometry.
}
    \label{fig:noise_vis}
\end{figure*}

\begin{figure*}[t]
    \centering
    \includegraphics[width=0.9\linewidth]{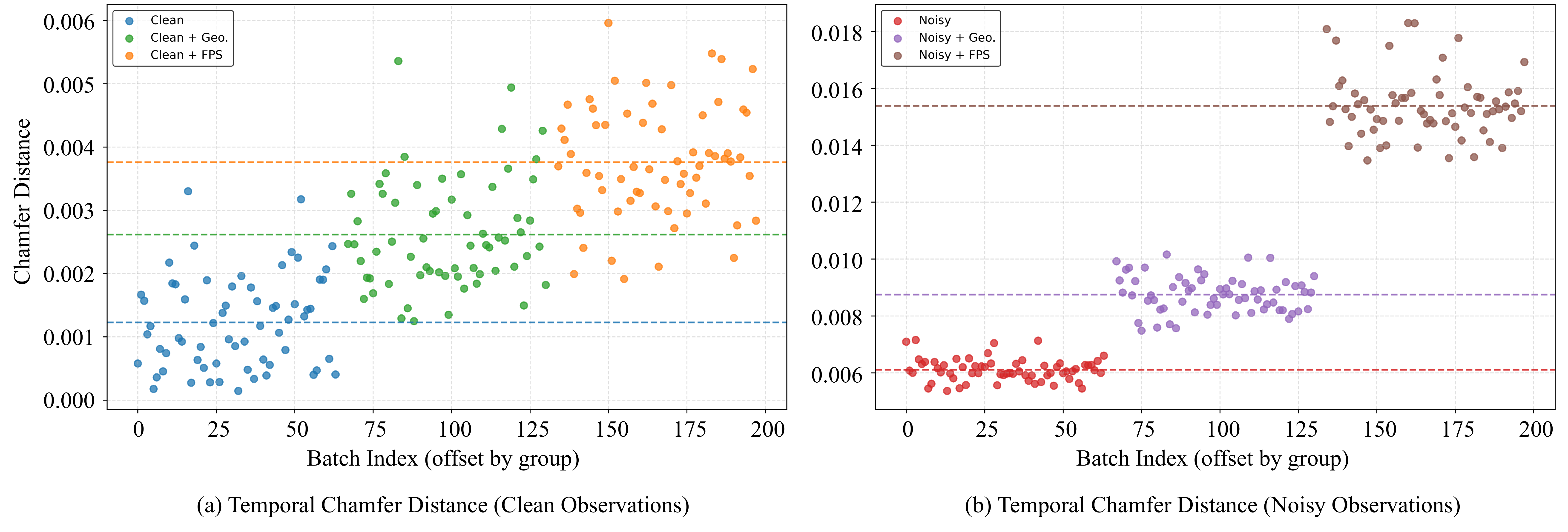}
    \caption{Temporal geometric consistency measured by frame-to-frame Chamfer distance. (a) Clean observations ($\sigma=0.00$). (b) Noisy observations ($\sigma=0.08$). For each batch element, the Chamfer distance is computed between two adjacent point cloud frames. Lower values correspond to improved temporal geometric consistency.
    Dashed horizontal lines indicate the mean Chamfer distance for each group.}
    \label{fig:noise_scatter}
\end{figure*}

\begin{figure*}[t]
    \centering
    \includegraphics[width=0.9\linewidth]{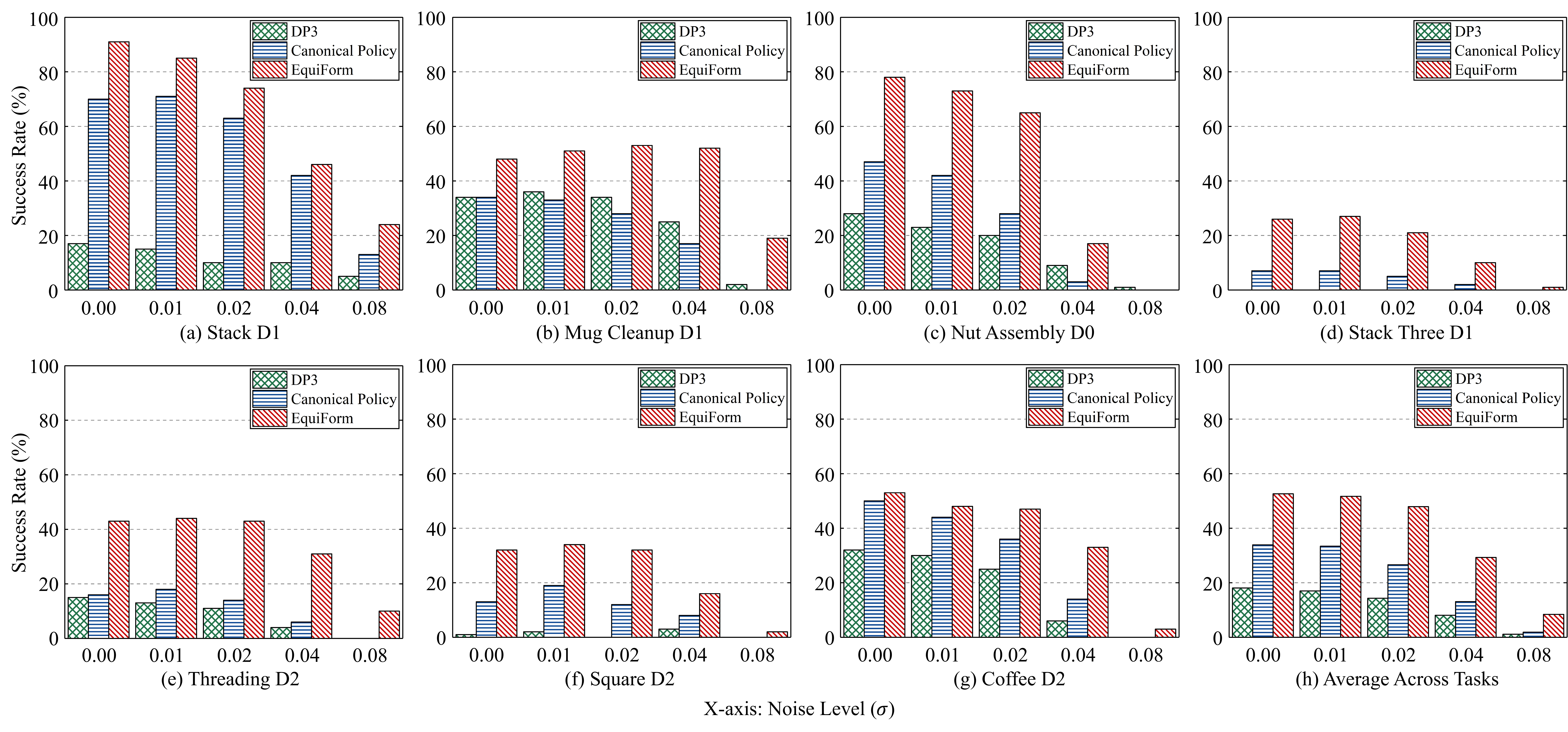}
    \caption{Robustness evaluation on challenging MimicGen tasks under increasing Gaussian observation noise. Subfigures (a)–(g) report task success rates on selected challenging benchmarks, while (h) shows the average success rate across all tasks. Gaussian noise with standard deviation $\sigma$ is added isotropically to point cloud observations.}
    \label{fig:noise_bar}
\end{figure*}

\begin{figure*}[t]
    \centering
    \includegraphics[width=0.9\linewidth]{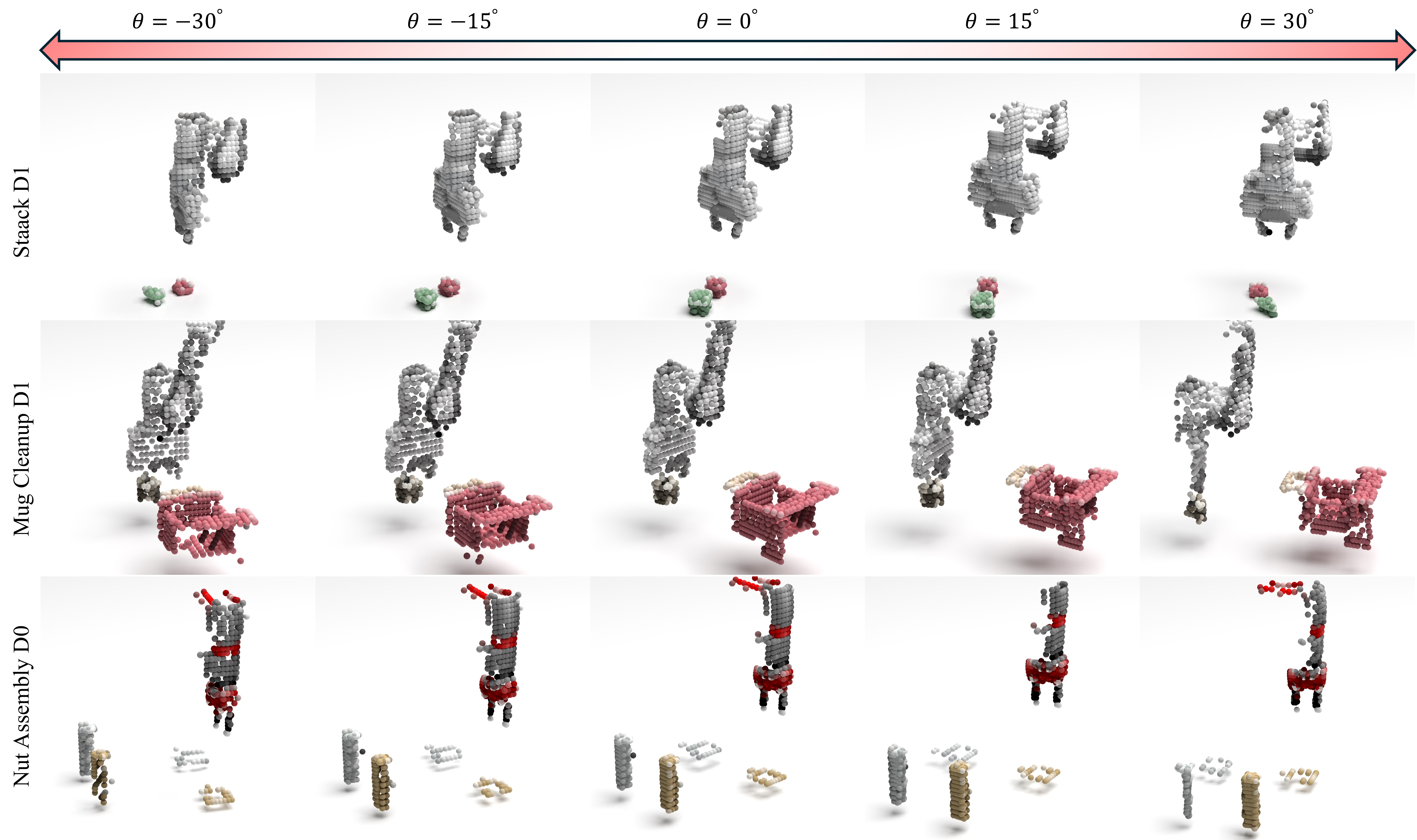}
    \caption{Qualitative visualization of point cloud observations under rigid and partial observation perturbations. Observations are rotated in-plane by angle $\theta$ (left to right) and further affected by random translation and partial visibility. Each row shows a different task, illustrating the impact of rotation on observation geometry.
}
    \label{fig:rot_vis}
\end{figure*}

\begin{figure*}[t]
    \centering
    \includegraphics[width=0.9\linewidth]{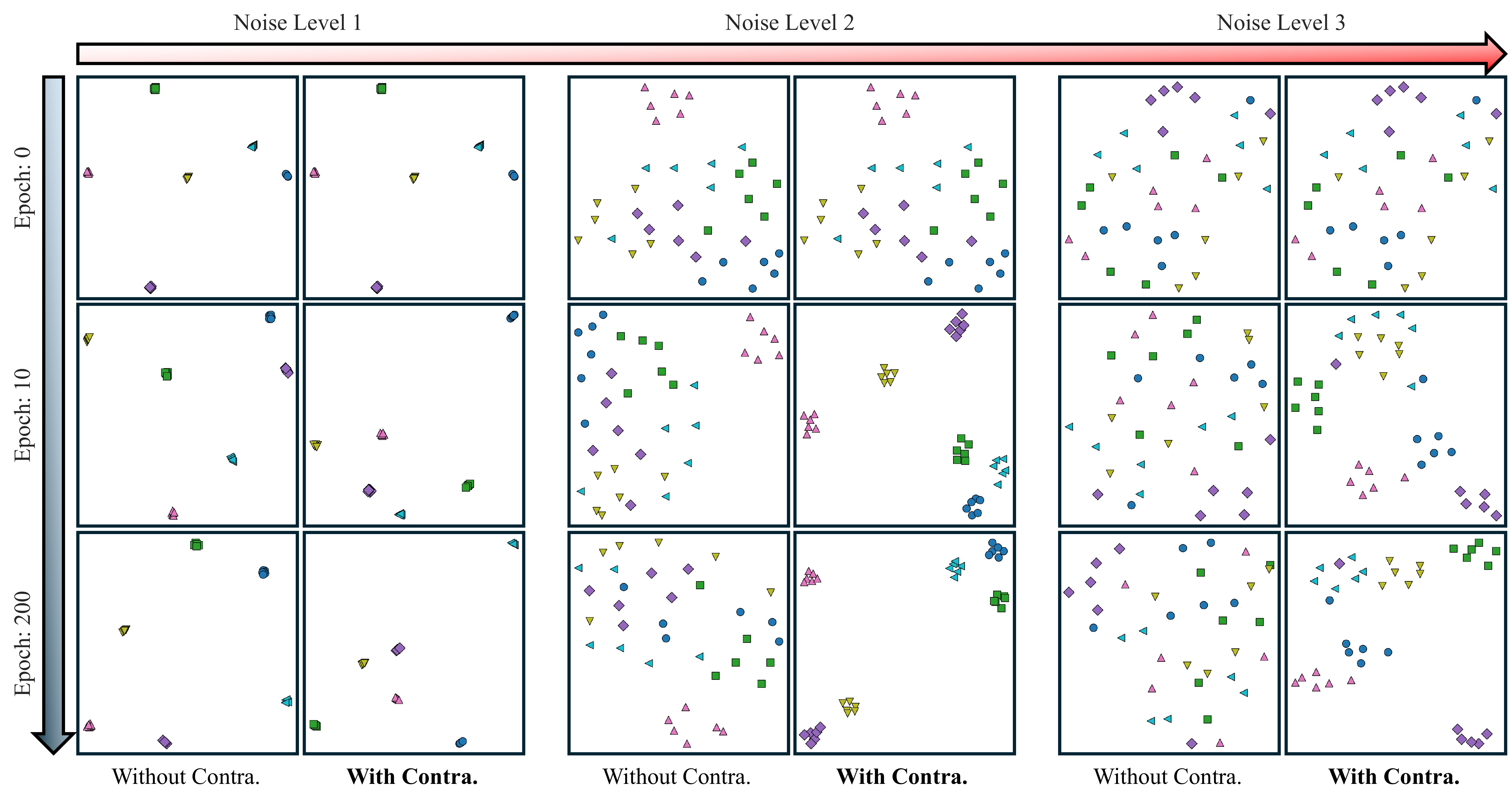}
    \caption{Robustness of equivariant representations under increasing observation noise. Equivariant feature embeddings are visualized across training epochs and noise levels, comparing models trained with and without contrastive learning.}
    \label{fig:vis_feat}
\end{figure*}

\begin{figure*}[t]
    \centering
    \includegraphics[width=0.9\linewidth]{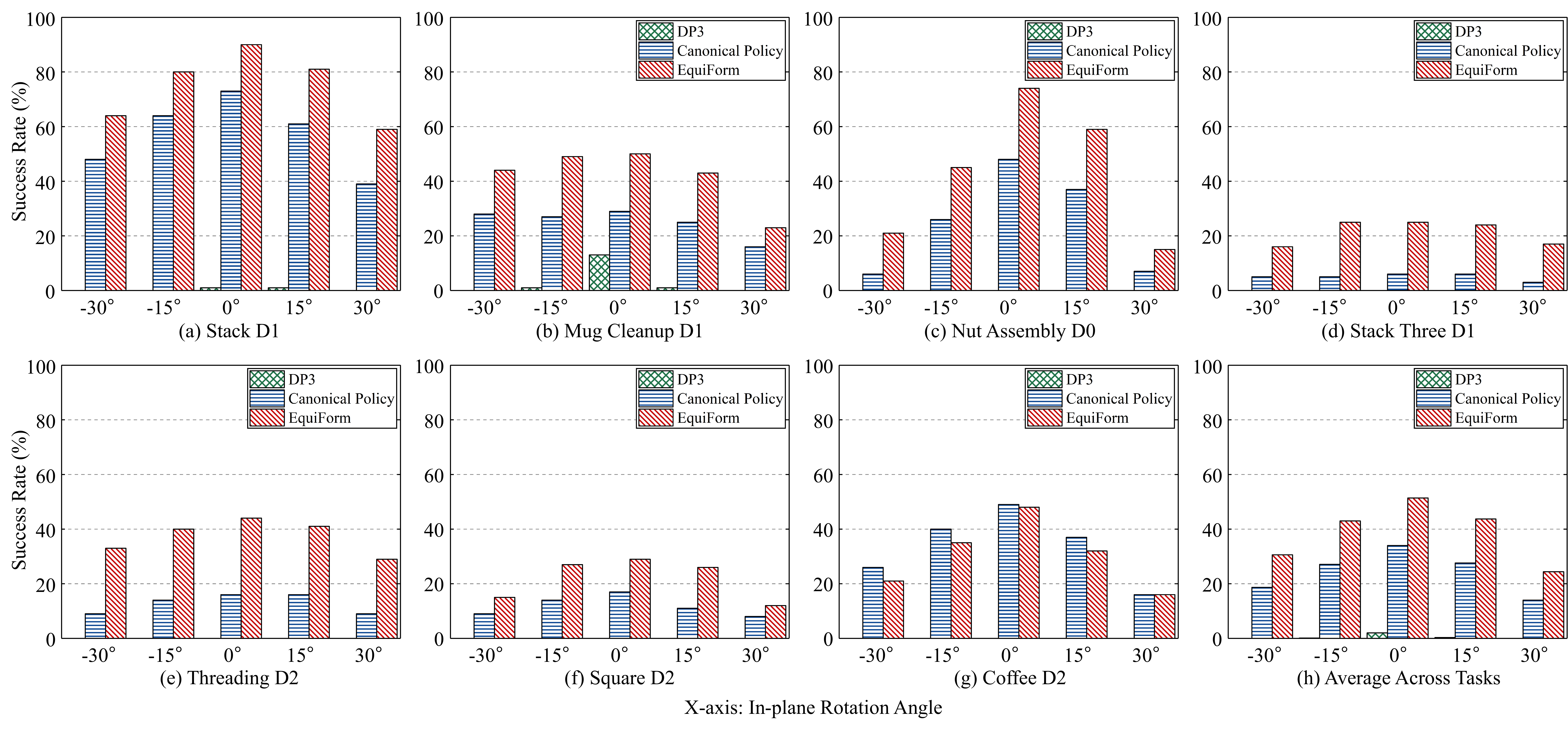}
    \caption{Robustness under rigid observation perturbations on challenging MimicGen tasks. Success rates are reported as a function of the in-plane rotation angle applied to point cloud observations. In addition to rotation, point clouds are randomly translated and subjected to partial observation effects, including cropping, dropout, and point insertion. Subfigures (a)–(g) show per-task performance, while (h) reports the average across tasks.}
    \label{fig:rot_bar}
\end{figure*}

\begin{table*}[t]
\caption{Ablation study on geometric denoising (Geo.) and contrastive learning (Contra.). Task success rates (\%).}
\label{tab:ablation}
\centering
\renewcommand{\arraystretch}{1.15}
\setlength{\tabcolsep}{5pt}

\newcommand{\tblwidth}{0.915\textwidth}

\begin{tabularx}{\tblwidth}{cccccccc}
\toprule
Geo. & Contra. &
\textbf{Stack D1} & \textbf{Mug Cleanup D1} & \textbf{Nut Assembly D0} &
\textbf{Stack Three D1} & \textbf{Threading D2} & \textbf{Square D2} \\
\midrule
$\times$ & $\times$ & 76 & 44 & 38 & 10 & 15 & 14 \\
$\checkmark$ & $\times$ & 94 & 48 & 75 & 23 & 37 & \textbf{41} \\
$\times$ & $\checkmark$ & 75 & 39 & 61 & 3 & 31 & 24 \\
\rowcolor{gray!30}
$\checkmark$ & $\checkmark$ & \textbf{96} & \textbf{50} & \textbf{76} & \textbf{25} & \textbf{42} & 33 \\

\midrule
Geo. & Contra. &
\textbf{Coffee D2} & \textbf{PushT} & \textbf{Beat Block Ham.} &
\textbf{Click Alarmclock} & \textbf{Handover Block} & \textbf{Move Can Pot} \\
\midrule
$\times$ & $\times$ & 52 & 87 & 93 & 65 & 68 & 84 \\
$\checkmark$ & $\times$ & 52 & -- & 93 & 86 & 95 & 89 \\
$\times$ & $\checkmark$ & 45 & 96 & \textbf{98} & 88 & 39 & 89 \\
\rowcolor{gray!30}
$\checkmark$ & $\checkmark$ & \textbf{53} & \textbf{96} & 87 & \textbf{89} & 85 & \textbf{90} \\

\midrule
Geo. & Contra. &
\textbf{Pick Dual Bottles} & \textbf{Place Empty Cup} &
\textbf{Stamp Seal} & \textbf{Turn Switch} &
\multicolumn{2}{c}{\textbf{Average Success Rate}} \\
\midrule
$\times$ & $\times$ & \textbf{96} & 81 & 34 & 46 & \multicolumn{2}{c}{56.4} \\
$\checkmark$ & $\times$ & 85 & 85 & \textbf{50} & 36 & \multicolumn{2}{c}{65.9} \\
$\times$ & $\checkmark$ & 88 & \textbf{90} & 9 & 40 & \multicolumn{2}{c}{57.2} \\
\rowcolor{gray!30}
$\checkmark$ & $\checkmark$ & 86 & 76 & 43 & \textbf{49} & \multicolumn{2}{c}{\textbf{67.3}} \\

\bottomrule
\end{tabularx}

\vspace{0.3em}
\footnotesize
\parbox{0.93\linewidth}{
We report the mean success rate across 16 simulation benchmarks.
“Geo.” denotes the proposed geometric denoising module, while “Contra.” denotes the use of contrastive representation learning.
Bold numbers indicate the best performance for each task.
}
\end{table*}
Table~\ref{tab:sim_comparison} summarizes the performance of different policies across 16 simulated manipulation tasks.
For the MimicGen tasks, we generate 200 demonstrations~\cite{equidiff} and train each method for 250 epochs.
For the Push-T task, which also contains 200 demonstrations~\cite{diffusion_policy}, each method is trained for 1000 epochs.
For the RoboTwin tasks, all policies are trained on the Aloha-AgileX embodiment using 50 demonstrations per task and trained for 200 epochs, consistent with the original RoboTwin setup~\cite{chen2025robotwin}.

For the MimicGen and Push-T tasks, all experiments are conducted using three random environment seeds (42, 43, and 44).
For each seed, policies are evaluated over the final 10 training epochs, and each evaluation rollout is conducted across 50 different environment
initializations, with results averaged accordingly.
For the RoboTwin tasks, we adopt the evaluation protocol from the original RoboTwin work~\cite{chen2025robotwin}, where each policy is evaluated over 100 independent trials under the demo\_clean setting with seed 0.
The mean success rate reported in Table~\ref{tab:sim_comparison} is computed by averaging results across all evaluation runs, and the best performance for each task is highlighted in bold.

We now analyze the results in Table~\ref{tab:sim_comparison} across different task families, focusing on how task structure and observation characteristics influence policy performance.

RoboTwin comprises bimanual manipulation tasks with relatively structured environments and constrained object configurations. In this setting, EquiForm demonstrates competitive performance, achieving the best results on the majority of tasks. On average, EquiForm improves over the non-equivariant baseline DP3 by 15.8\%, and over the $\mathrm{SE}(3)$-equivariant Canonical Policy by 4.7\%. However, Canonical Policy outperforms EquiForm on a small subset of tasks, including Beat Block Hammer, Pick Dual Bottles, and Place Empty Cup. These tasks exhibit limited global rigid-body pose variation across episodes and instead emphasize local symmetries, tool-use interactions, or coordinated bimanual motion. In such cases, the geometric assumptions underlying Canonical Policy are already well aligned with the task structure, reducing the benefit of additional robust canonicalization.

For the 3D Push-T task, observations consist of only nine sparse keypoints, and geometric denoising is not applicable. We therefore apply contrastive learning to improve observation canonicalization. Under this setting, EquiForm outperforms DP3 by 22.0\% and Canonical Policy by 4.0\%, indicating that robust canonicalization remains effective even under highly sparse and abstract observations.

In contrast, EquiForm exhibits its most pronounced performance gains on the single-arm MimicGen benchmark, which features substantial pose variation, cluttered scenes, and noisy point cloud observations. Averaged across all MimicGen tasks, EquiForm outperforms DP3 by 34.5\% and exceeds Canonical Policy by 17.0\%. In these settings, pose diversity and observation noise significantly hinder policy learning. By canonicalizing observations into a shared reference frame and incorporating noise-robust learning, EquiForm effectively reduces pose-induced variability and stabilizes representations under partial and noisy observations, enabling the policy to focus on task-relevant geometry.

When averaged across all evaluated tasks, EquiForm achieves an overall success rate of 66.6\%, representing the highest average performance among all methods. This corresponds to a 10.0\% improvement over Canonical Policy and a 24.3\% improvement over DP3, highlighting the effectiveness of EquiForm in generalizing across task families with varying degrees of pose variability and observation noise.

\subsection{Robustness under Noisy Inputs}

While the above results demonstrate strong task-level performance,
they do not isolate the effect of observation noise on geometric
consistency.
In practice, point cloud observations are often corrupted by sensor
noise and sampling variability, which can lead to frame-to-frame
inconsistency and destabilize policy execution.
To explicitly evaluate the noise robustness of EquiForm and validate the
design of the geometric denoising module, we conduct a controlled study
by injecting isotropic Gaussian noise into point cloud observations.

Fig.~\ref{fig:noise_vis} provides a qualitative visualization of the
effect of increasing Gaussian noise on point cloud observations.
As the noise level increases from left to right, raw noisy observations
exhibit severe surface distortion and loss of coherent geometric
structure.
FPS, while effective for spatial subsampling, fails to recover meaningful surface geometry under strong noise, leading to fragmented and inconsistent point distributions.
In contrast, the proposed geometric denoising module preserves surface
continuity and spatial coherence even under severe noise.
This qualitative result highlights that geometric denoising addresses
noise-induced surface corruption at the geometry level, rather than
relying on task-specific heuristics or policy-level adaptation.

Fig.~\ref{fig:noise_scatter} reports the temporal Chamfer distance~\cite{chamfer} between
adjacent point cloud frames under both clean and noisy observation
settings.
In our simulation setup, point cloud observations are sampled at 10\,Hz,
and the robot motion between consecutive time steps is small.
As a result, adjacent frames are expected to exhibit high geometric
similarity.
Fig.~\ref{fig:noise_scatter}(a) corresponds to the clean observation setting ($\sigma=0.0$).
Under this setting, raw point cloud observations achieve the highest
geometric similarity.
However, FPS reduces temporal consistency
due to inconsistent sampling across frames.
In contrast, the proposed geometric denoising module maintains higher
frame-to-frame similarity by enforcing more uniform point distributions
through the tangent-direction correction, resulting in improved sampling consistency over time.
Fig.~\ref{fig:noise_scatter}(b) shows the noisy observation setting ($\sigma=0.08$).
Under strong noise, temporal inconsistency becomes substantially more
pronounced.
While FPS further amplifies frame-to-frame discrepancy, the geometric denoising module preserves higher temporal
similarity.
This robustness arises from the combined effects of the normal-direction correction, which constrains points to the underlying surface, and the tangent-direction correction, which promotes uniform sampling across frames.
These results demonstrate that geometric denoising effectively mitigates sampling inconsistency under noisy observations, directly validating the design motivation of the proposed module.

Finally, we examine how improved geometric stability under noise
translates into downstream task performance.
Fig.~\ref{fig:noise_bar} reports task success rates on selected
challenging MimicGen tasks as well as the average performance across all
tasks, under increasing levels of Gaussian observation noise.
While baseline methods experience rapid performance degradation as noise
intensifies, EquiForm maintains consistently higher success rates across all
noise levels.
This trend is especially pronounced on tasks with lower baseline success
rates, where observation noise and pose variation more severely hinder
policy learning.

Taken together, these results demonstrate that the geometric denoising
module effectively stabilizes point cloud observations under noisy inputs,
leading to improved temporal consistency and, ultimately, more robust
policy execution.

\subsection{Robustness under Rigid Transformations}

In addition to observation noise, point cloud observations are often subject to rigid transformations caused by viewpoint or layout changes, which are frequently accompanied by partial observability and uneven point distributions, thereby severely degrading equivariant representations.
Fig.~\ref{fig:rot_vis} provides a qualitative visualization of point cloud
observations under increasing in-plane rotations.
Observations are rotated by angle $\theta$ and further perturbed by random
translation and partial visibility, including cropping, point dropout,
and point insertion.
The severity of occlusion increases with the magnitude of the rotation,
reflecting the fact that larger viewpoint changes typically induce greater
partial observability.

We design an experiment to evaluate the robustness of the equivariant feature representations introduced in Section~\ref{sec:contrastive} under rigid observation perturbations.
Specifically, we sample six point cloud observations from the Stack~D1
task and apply six augmentations per sample under three increasing
perturbation levels:
\begin{itemize}
    \item \textbf{Level 1}: random in-plane rotation only;
    \item \textbf{Level 2}: random rotation combined with pointwise Gaussian
    jitter ($\mu=0,\;\sigma=0.1$) and random cropping, dropout, and insertion,
    where each operation affects 10\% of the points;
    \item \textbf{Level 3}: same with $\sigma=0.2$ and 20\% perturbations.
\end{itemize}

We compare equivariant feature embeddings produced by models trained
\textit{with} and \textit{without} equivariant contrastive learning.
UMAP~\cite{umap} is used to project the resulting equivariant features
$\mathbf{z}$ into a two-dimensional space for visualization.

As shown in Fig.~\ref{fig:vis_feat}, at Noise Level~1, where only rigid
rotations are applied, equivariant features corresponding to the same
point cloud instance cluster tightly for both settings, reflecting the
inherent equivariance of the network.
At Noise Level~2, where rotation is accompanied by moderate noise and
partial observation, features produced without contrastive learning
become increasingly scattered, indicating that equivariance alone is
insufficient to ensure robustness under coupled perturbations.
In contrast, models trained with contrastive learning preserve compact
and coherent feature clusters, with intra-instance consistency improving
progressively over training epochs.
At Noise Level~3, corresponding to strong noise and severe partial
observations, the equivariant structure of features without contrastive
learning largely collapses, whereas contrastive learning still enables
the network to maintain structured and distinguishable clusters.

We further evaluate the impact of rigid observation perturbations on
downstream task performance. Fig.~\ref{fig:rot_bar} reports success rates on representative MimicGen tasks as a function of the in-plane rotation angle $\theta$ applied to point cloud observations.
In addition to rotation, point clouds are also subjected to random
translation and partial observation effects, including cropping, dropout,
and point insertion, to reflect realistic sensing conditions.
Notably, DP3 suffers a near-complete performance collapse even at
$\theta=0$, where observations are only translated but not rotated.
This behavior highlights the lack of any built-in equivariance in DP3,
making it highly sensitive to rigid transformations.
In contrast, both the Canonical Policy and EquiForm maintain relatively
stable performance under rigid perturbations, as they explicitly model
equivariance with respect to rigid transformations.
As the rotation magnitude increases and partial observations become more severe, the performance of both methods degrades across tasks.
However, EquiForm achieves higher success rates than the Canonical Policy across a wide range of rotation angles.
This improvement can be attributed to equivariant contrastive learning,
which stabilizes equivariant features under coupled rigid transformations
and partial observability, resulting in a more reliable canonical
representation.

Overall, these results suggest that equivariance alone is insufficient to
ensure robustness under realistic rigid perturbations.
By combining geometric denoising with equivariant contrastive learning,
EquiForm more reliably canonicalizes approximately rigid and partially
observed point clouds, thereby maintaining stable policy execution under
viewpoint and layout changes.

\subsection{Ablation Study}
We conduct an ablation study in Table~\ref{tab:ablation} to evaluate the individual and combined effects of geometric denoising (Geo.) and equivariant contrastive learning (Contra.) on policy performance.
Overall, enabling either geometric denoising or equivariant contrastive learning improves performance on a subset of tasks, indicating that both components provide useful but distinct contributions.
When both components are combined, the full EquiForm model achieves the highest average success rate of 67.3\% across all benchmarks, demonstrating that the two modules are complementary at an aggregate level.

We note, however, that the combined model does not achieve the best performance on every individual task.
In particular, for tasks such as Pick Dual Bottles, the variant without geometric denoising and contrastive learning attains higher success rates.
These tasks exhibit relatively constrained scene layouts and limited pose variation, where strong canonicalization or geometric smoothing may suppress fine-grained, task-relevant spatial details.

In contrast, tasks involving larger pose variability, partial observations, or cluttered interactions, such as Stack, Nut Assembly, and Threading tasks in MimicGen, benefit substantially from the joint use of geometric denoising and equivariant contrastive learning.
In such settings, the two components jointly stabilize the canonical representation under observation noise and rigid transformations, leading to more robust downstream policy performance.

Taken together, these results suggest that geometric denoising and equivariant contrastive learning contribute complementary benefits rather than redundant functionality, and that their effectiveness is task-dependent, with the largest gains arising in scenarios where pose variation and observation noise pose significant challenges.

\section{Real-Robot Experiments}
In this section, we evaluate the proposed EquiForm policy on real-world
Franka Emika Panda robotic platforms.
We describe the experimental setup and benchmark EquiForm against
point cloud--based baselines on a set of diverse manipulation tasks.
We further assess robustness and generalization under scene-level
$\mathrm{SE}(3)$ layout variations, followed by an analysis of
representative failure cases.

\subsection{Experimental Setup}
\begin{figure*}[t]
    \centering
    \includegraphics[width=1.0\linewidth]{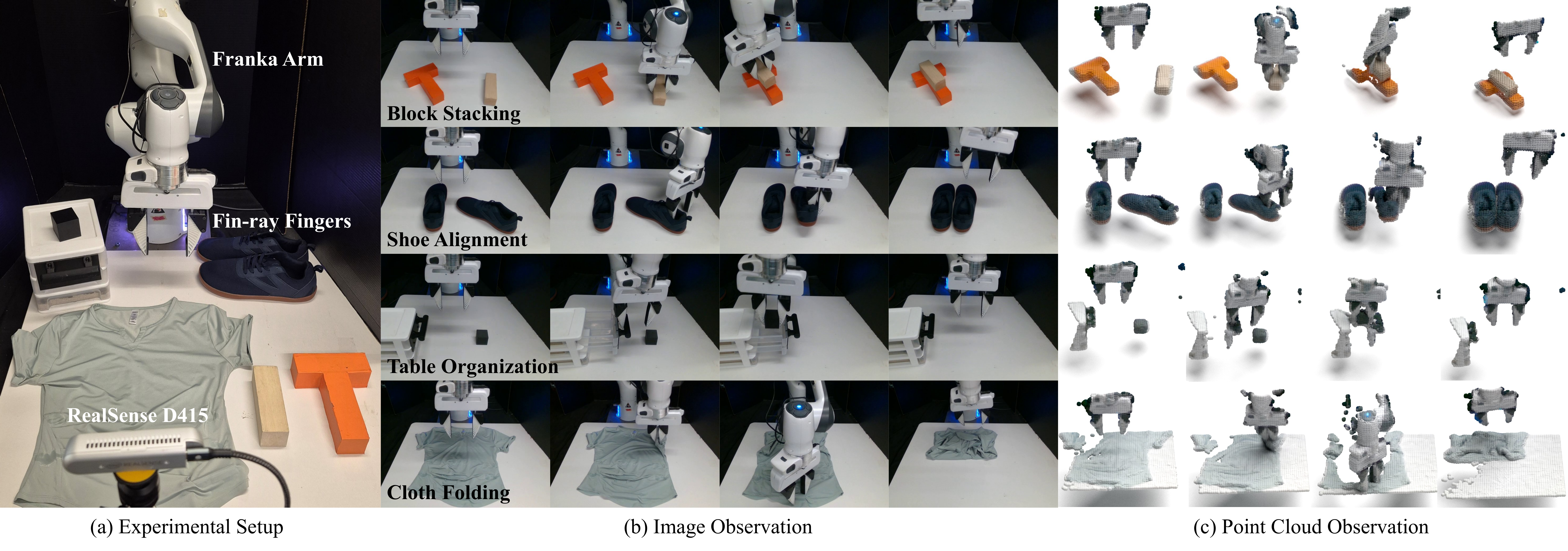}
    \caption{Experimental setup and visual observations. (a) The real-world robotic manipulation setup, consisting of a Franka arm equipped with Fin-Ray fingers and an Intel RealSense D415 camera. (b) Image observations captured over time for representative manipulation tasks, including block stacking, shoe alignment, table organization, and cloth folding. (c) Corresponding point cloud observations reconstructed from RGB-D images.}
    \label{fig:realworld}
\end{figure*}
\begin{table}[t]
\caption{Task success rates (\%) across different policies and tasks under standard rollout settings.}
\label{tab:realworld_normal}
\centering
\renewcommand{\arraystretch}{1.1}
\setlength{\tabcolsep}{2pt}

\newcolumntype{C}{>{\centering\arraybackslash}X}

\begin{tabularx}{\columnwidth}{lCCCC}
\toprule
 & Block Stacking & Shoe Alignment & Table Organization & Cloth Folding \\
\midrule
DP3 & 1/10 & 2/10 & 0/10 & 1/10 \\
Canonical Policy & 3/10 & 3/10 & 1/10 & 0/10 \\
\rowcolor{gray!30}
EquiForm & \textbf{5/10} & \textbf{8/10} & \textbf{2/10} & \textbf{3/10} \\
\bottomrule
\end{tabularx}
\end{table}

\begin{figure}[t]
    \centering
    \includegraphics[width=1.0\linewidth]{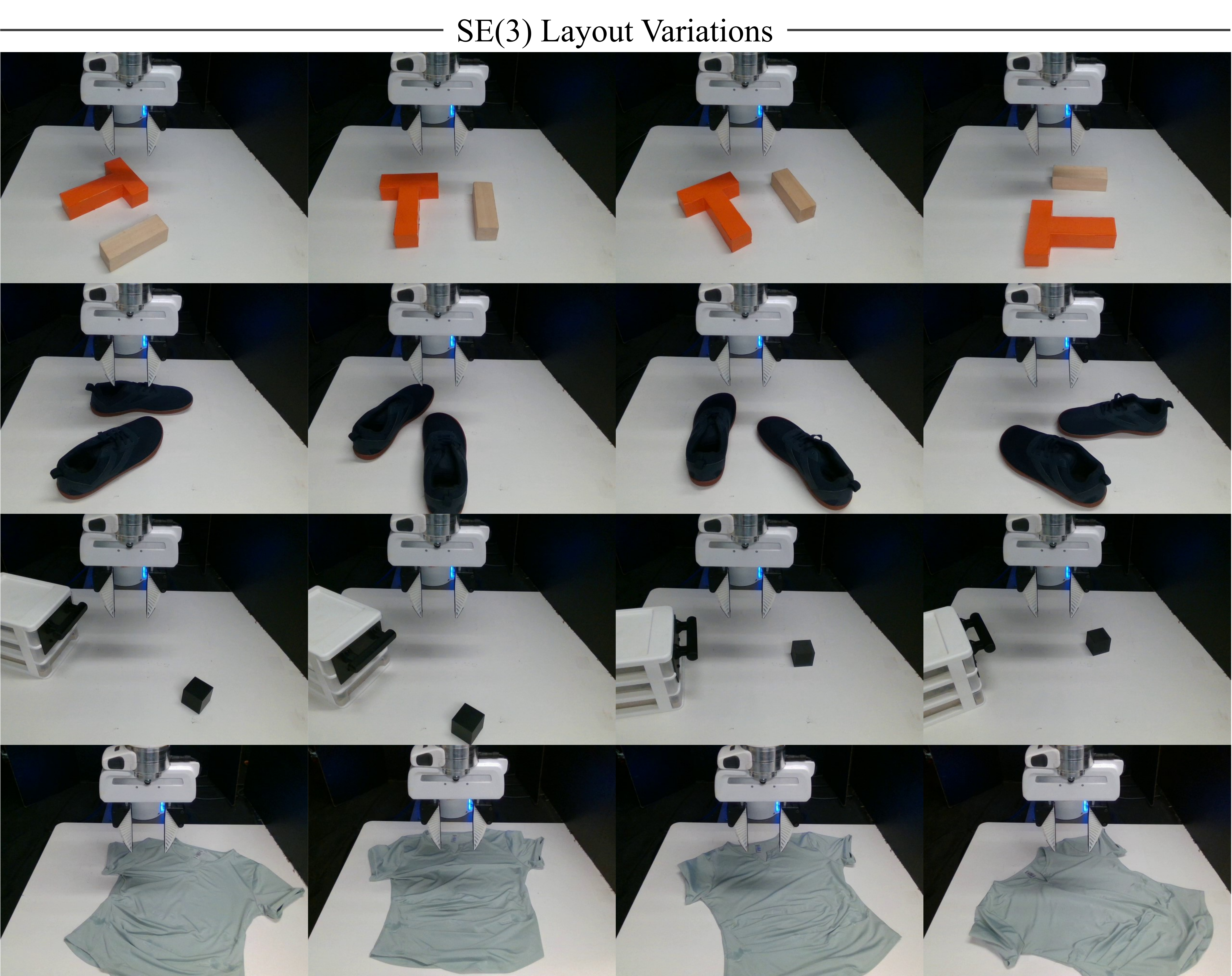}
    \caption{SE(3) layout variations in real-world scenes. We visualize initial observations from real-robot manipulation tasks under diverse scene-level translations and rotations. Each row corresponds to one task, while columns show different randomized scene configurations.}
    \label{fig:se3_layout}
\end{figure}

\begin{figure*}[t]
    \centering
    \includegraphics[width=0.8\linewidth]{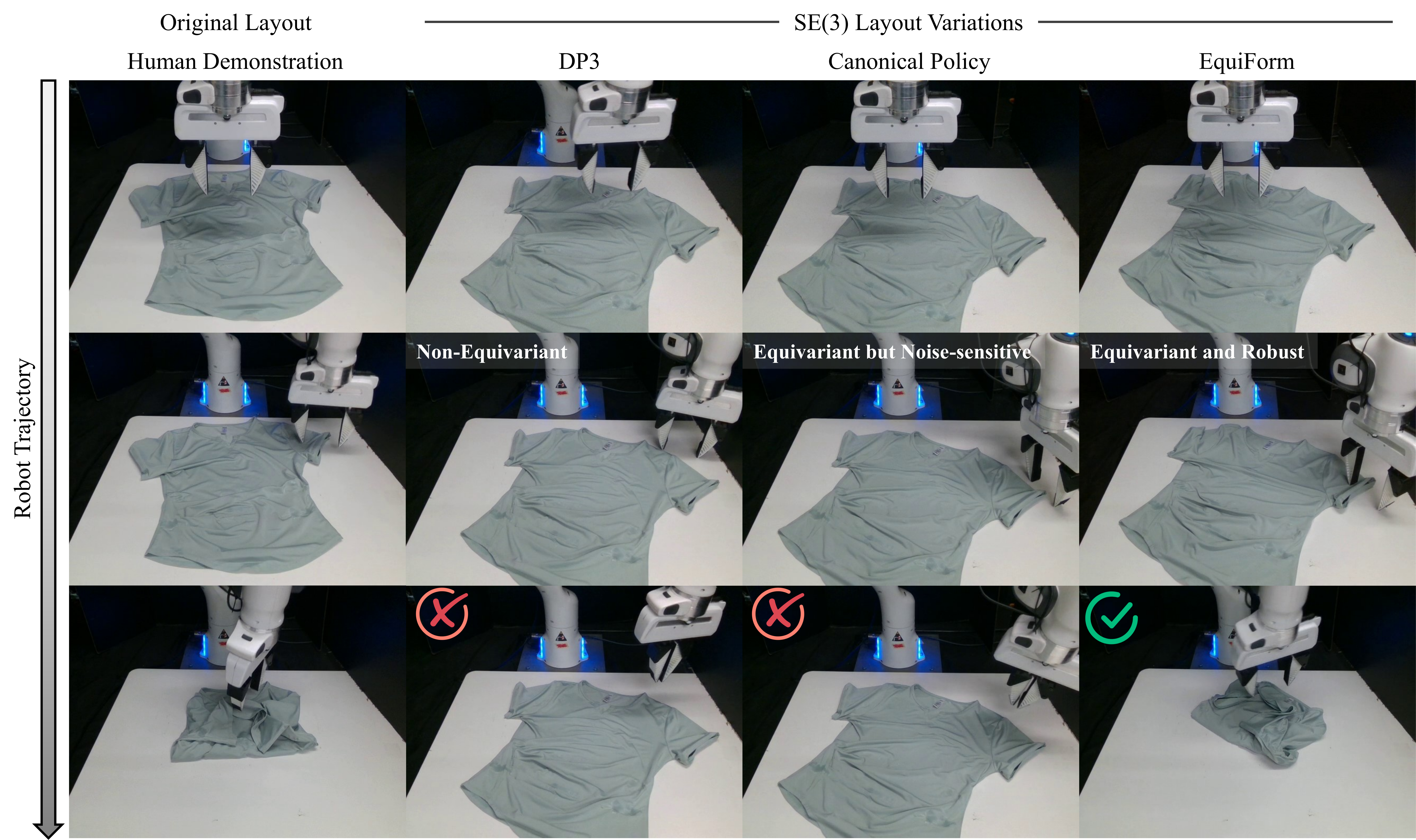}
    \caption{SE(3) layout variations in real-world scenes. The leftmost column shows a human demonstration as a reference, while the remaining columns compare policy rollouts under SE(3) layout variations.}
    \label{fig:se3_layout_vis}
\end{figure*}

\begin{table}[t]
\caption{Task success rates (\%) across different policies and tasks under SE3(3) layout variations.}
\label{tab:realworld_se3}
\centering
\renewcommand{\arraystretch}{1.1}
\setlength{\tabcolsep}{2pt}

\newcolumntype{C}{>{\centering\arraybackslash}X}

\begin{tabularx}{\columnwidth}{lCCCC}
\toprule
 & Block Stacking & Shoe Alignment & Table Organization & Cloth Folding \\
\midrule
DP3 & 1/10 & 3/10 & 0/10 & 0/10 \\
Canonical Policy & 3/10 & 3/10 & 0/10 & 0/10 \\
\rowcolor{gray!30}
EquiForm & \textbf{5/10} & \textbf{6/10} & \textbf{2/10} & \textbf{2/10} \\
\bottomrule
\end{tabularx}
\end{table}
We evaluate the proposed EquiForm policy on a real-world robotic manipulation platform. The experimental setup and representative visual observations are shown in Fig.~\ref{fig:realworld}.
The system consists of a Franka Emika Panda~\cite{panda} arm equipped with a pair of Fin-Ray fingers~\cite{finray} for compliant grasping and an Intel RealSense D415~\cite{realsense} RGB-D camera for visual sensing. Human demonstrations are collected using a 6-DoF 3DConnexion SpaceMouse~\cite{spacemouse}, with observations and actions recorded at 10~Hz. During policy rollout, DDIM~\cite{DDIM} sampling with 20 denoising steps is used for efficient inference.

RGB-D observations captured by the D415 camera are converted into 3D point clouds using calibrated camera intrinsics. Point clouds are spatially cropped to a fixed workspace region and retain only the $\mathit{xyz}$ coordinates, discarding color information to improve robustness to appearance variation. Consistent with the simulation experiments, each point cloud is uniformly downsampled to 256 points for efficient inference. Representative image observations and their corresponding point cloud reconstructions are shown in Fig.~\ref{fig:realworld}(b) and Fig.~\ref{fig:realworld}(c), respectively.

We evaluate four real-world manipulation tasks with increasing geometric and interaction complexity: Block Stacking, Shoe Alignment, Table Organization, and Cloth Folding.
The first two tasks primarily involve rigid objects and short-horizon
interactions, while Table Organization and Cloth Folding require
long-horizon reasoning under cluttered scenes.
In particular, Cloth Folding involves non-rigid deformation and thin
structures that are difficult to observe reliably, leading to increased
observation noise and geometric ambiguity in point cloud representations.
We compare EquiForm against representative point cloud--based baselines,
including DP3 and the Canonical Policy, with all methods operating on
identical point cloud observations and task definitions.

\subsection{Benchmarking on Diverse Tasks}
Table~\ref{tab:realworld_normal} summarizes the quantitative performance of
different point cloud--based policies across four real-world manipulation
tasks on the Franka Emika Panda platform, evaluated over 10 rollout trials
per task. Overall, EquiForm consistently achieves the highest success rates
across all tasks compared to DP3 and the Canonical Policy, demonstrating
improved robustness in real-world settings.

For the relatively short-horizon tasks involving rigid objects, namely
Block Stacking and Shoe Alignment, EquiForm attains success rates of 50\% and
80\%, respectively. These results clearly outperform DP3, which fails in
most trials, and improve upon the Canonical Policy, consistent with more
stable canonicalization under real-world sensing noise and execution
variability.  
On long-horizon and geometrically challenging tasks, including Table
Organization and Cloth Folding, performance differences remain observable.
Both DP3 and the Canonical Policy succeed in at most a single trial under
these settings, indicating limited robustness in real-world scenarios.
In contrast, under the same evaluation protocol, EquiForm completes these
tasks in multiple trials, indicating improved robustness under real-world
conditions.

\subsection{Validation under Rigid Transformations}

We further evaluate robustness under scene-level rigid transformations by
introducing $\mathrm{SE}(3)$ layout variations in real-world environments.
Specifically, we apply random translations and in-plane rotations to the
initial object layouts while keeping the robot and sensing setup unchanged.
Representative examples of these layout variations are shown in
Fig.~\ref{fig:se3_layout}, where each row corresponds to a task and columns
illustrate different randomized scene configurations.

Table~\ref{tab:realworld_se3} reports task success rates under these layout
variations. For Block Stacking and Shoe Alignment, both DP3 and the
Canonical Policy retain limited performance, while EquiForm achieves higher
success rates across both tasks. On the more challenging Table
Organization and Cloth Folding tasks, DP3 and the Canonical Policy fail in
all trials. In contrast, EquiForm continues to achieve successful task
executions in multiple trials, exhibiting a performance trend consistent
with that observed under the standard rollout setting in Table~\ref{tab:realworld_normal}.

Fig.~\ref{fig:se3_layout_vis} presents qualitative rollouts under $\mathrm{SE}(3)$ layout variations for the Cloth Folding task, which is particularly challenging due to non-rigid deformation and noisy, irregular point cloud geometry.
These characteristics make the task highly sensitive to layout changes and sampling instability, providing a representative testbed for
evaluating robustness under coupled rigid transformations and observation
noise.

As shown in Fig.~\ref{fig:se3_layout_vis}, DP3, which lacks equivariant
structure, follows a fixed trajectory pattern learned from demonstrations.
When the scene is transformed, this behavior results in grasps being
executed at incorrect spatial locations.
By contrast, the Canonical Policy exhibits equivariant behavior: when the
cloth is rotated, the grasping trajectory rotates accordingly.
However, due to its sensitivity to noisy point cloud observations, the
estimated grasp locations become inaccurate, leading to failed contacts
and unsuccessful grasps.
In contrast, EquiForm preserves robust equivariant behavior under the same
conditions.
By stabilizing canonical representations through geometric denoising and
equivariant contrastive learning, EquiForm produces grasp trajectories that
transform consistently with the scene while remaining accurate under noisy
observations.

\subsection{Failure Case Analysis}
\begin{figure*}[t]
    \centering
    \includegraphics[width=0.8\linewidth]{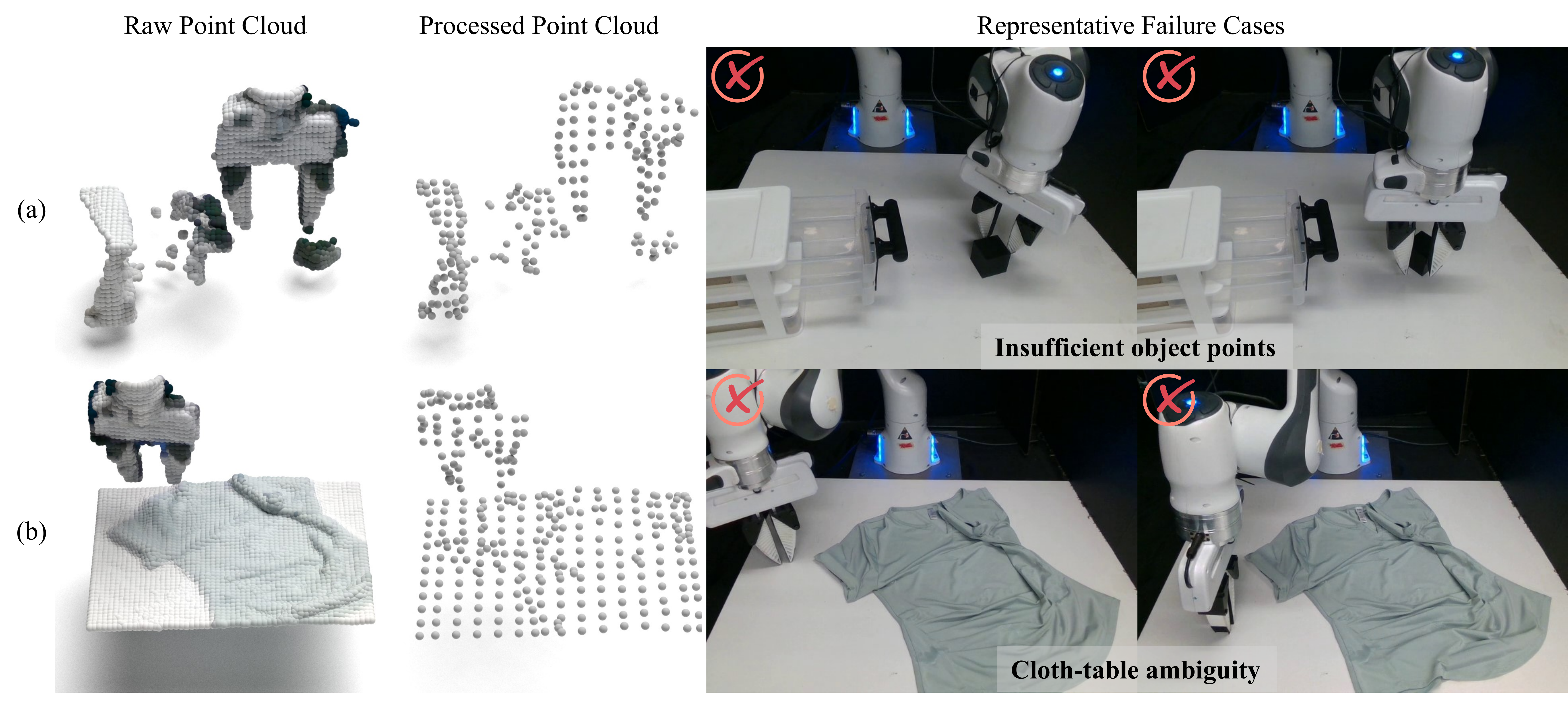}
    \caption{Representative failure cases in real-world experiments. Left: raw RGB-D point clouds and the corresponding processed point clouds after downsampling and color removal. Right: example failure cases in table organization and cloth folding. In table organization, small objects are under-represented after point cloud processing, resulting in insufficient geometric support for reliable grasping. In cloth folding, the thin structure of the cloth leads to ambiguity between the cloth and the supporting table in geometry-only point clouds.}
    \label{fig:failure_case}
\end{figure*}
We analyze representative failure cases to better understand the remaining
limitations of EquiForm in real-world settings.
Fig.~\ref{fig:failure_case} visualizes spatially cropped raw point clouds
alongside their processed representations after color removal and
downsampling, together with corresponding failure examples observed during
task execution.

In the Table Organization task, small objects are often under-represented
after point cloud processing, resulting in insufficient geometric support
for reliable grasp planning and execution.
As illustrated in Fig.~\ref{fig:failure_case}(a), uniform downsampling can
remove critical object details, which in turn leads to failed grasps or
incorrect object interactions.
In the Cloth Folding task, shown in Fig.~\ref{fig:failure_case}(b), the thin
and deformable structure of the cloth introduces significant ambiguity
between the cloth surface and the supporting table when using
geometry-only point cloud observations.
This ambiguity can result in inaccurate grasp localization and subsequent
failures during folding execution.

These failure modes highlight inherent challenges of point cloud–based
perception in real-world manipulation, particularly for tasks involving
small objects or thin, non-rigid materials.
Addressing these limitations may require more adaptive point sampling
strategies that better preserve object-level geometry, as well as the
integration of complementary sensing modalities, such as tactile feedback,
to improve perception of thin and deformable objects.

\section{Conclusions and Limitations}
In this work, we proposed EquiForm, a noise-robust $\mathrm{SE}(3)$-equivariant
policy framework for point cloud--based robotic manipulation.
EquiForm integrates geometric denoising and equivariant contrastive learning
to stabilize canonical representations under observation noise and rigid
layout variations.
Through extensive simulation benchmarks and real-world robotic
experiments, we demonstrated that EquiForm improves the reliability of policy
execution across a diverse set of manipulation tasks, particularly in
settings involving noisy point cloud observations and scene-level
$\mathrm{SE}(3)$ transformations.
Both quantitative results and qualitative analyses show that stabilizing
equivariant representations is critical for maintaining consistent
behavior under realistic sensing and layout perturbations.

Despite these improvements, EquiForm has several limitations.
First, uniform point cloud downsampling can under-represent small objects,
removing fine-grained geometric details that are important for reliable
grasp planning.
Second, for tasks involving thin and non-rigid objects such as Cloth
Folding, geometry-only point cloud observations can introduce ambiguity
between the target object and supporting surfaces, making accurate grasp
localization difficult.
Finally, while EquiForm improves robustness under moderate observation noise
and layout variations, extreme occlusions and severe sensing artifacts
remain challenging.
Addressing these limitations may require more adaptive point sampling
strategies, higher-resolution or object-aware geometric representations,
and the integration of complementary sensing modalities, such as tactile
perception.

\bibliographystyle{IEEEtran}
\bibliography{references}

\vfill

\end{document}